\def\year{2017}\relax
\documentclass[letterpaper]{article}
\usepackage{aaai17}
\usepackage{times}
\usepackage{helvet}
\usepackage{courier}

\usepackage{amsmath}
\usepackage{amsthm}
\usepackage{amssymb}
\usepackage{cases}
\usepackage{proof}
\usepackage{times}

\usepackage{algorithm}
\usepackage{algpseudocode}
\makeatletter
\def\BState{\State\hskip-\ALG@thistlm}
\makeatother

\usepackage{enumerate}

\usepackage{bm, bbm}

\usepackage{graphicx}
\usepackage{subfigure}

\usepackage{booktabs}
\usepackage{multirow}
\usepackage{tablefootnote}

\usepackage{pgf,tikz}
\usetikzlibrary{shapes,automata,snakes,backgrounds,arrows}
\usepackage{pgfplots}
\usetikzlibrary{datavisualization}
\pgfplotsset{compat=1.12}

\newtheorem{thm}{Theorem}

\newtheorem{lem}{Lemma}

\newtheorem{asmp}{Assumption}

\newcommand{\myvec}{\mathrm{vec}}
\newcommand{\tr}{\mathrm{tr}}
\newcommand{\real}[1]{\mathbb{R}^{{#1}}}
\newcommand{\innerprod}[1]{\langle {#1} \rangle}
\newcommand{\norm}[1]{\lVert {#1} \rVert}

\newcommand{\rank}{\mathrm{rank}}

\begin{document}
%
\title{Expectile Matrix Factorization for Skewed Data Analysis}
\author{Rui Zhu$^1$, Di Niu$^1$, Linglong Kong$^2$, and Zongpeng Li$^3$\\
${}^1$ Department of Electrical and Computer Engineering, University of Alberta, \{\texttt{rzhu3, dniu}\}\texttt{@ualberta.ca}\\
${}^2$ Department of Mathematical and Statistical Sciences, University of Alberta, \texttt{lkong@ualberta.ca}\\
${}^3$ Department of Computer Science, University of Calgary, \texttt{zongpeng@ucalgary.ca}
}
\maketitle

\begin{abstract}
	Matrix factorization is a popular approach to solving matrix estimation problems based on partial observations. Existing matrix factorization is based on least squares and aims to yield a low-rank matrix to interpret the conditional sample means given the observations. However, in many real applications with skewed and extreme data, least squares cannot explain their central tendency or tail distributions, yielding undesired estimates. In this paper, we propose \emph{expectile matrix factorization} by introducing asymmetric least squares, a key concept in expectile regression analysis, into the matrix factorization framework. We propose an efficient algorithm to solve the new problem based on alternating minimization and quadratic programming. We prove that our algorithm converges to a global optimum and exactly recovers the true underlying low-rank matrices when noise is zero. For synthetic data with skewed noise and a real-world dataset containing web service response times, the proposed scheme achieves lower recovery errors than the existing matrix factorization method based on least squares in a wide range of settings.
\end{abstract}

\section{Introduction}
\label{sec:intro}

Matrix estimation has wide applications in many fields such as recommendation systems \cite{koren2009matrix}, network latency estimation \cite{liao2013dmfsgd}, computer vision \cite{chen2004recovering}, system identification \cite{liu2009interior}, etc. In these problems, a low-rank matrix $M^*\in \mathbb R^{m\times n}$ or a linear mapping $\mathcal A(M^*)$ from the low-rank matrix $M^*$ is assumed to underlie some possibly noisy observations, where $\mathcal A: \mathbb R^{m\times n}\rightarrow \mathbb R^p$. The objective is to recover the underlying low-rank matrix based on partial observations $b_i$, $i=1,\ldots,p$. For example, a movie recommendation system aims to recover all user-movie preferences based on the ratings between some user-movie pairs \cite{koren2009matrix,su2009survey}, or based on implicit feedback, e.g., watching times/frequencies, that are logged for some users on some movies \cite{hu2008collaborative,rendle2009bpr}.
In network or web service latency estimation \cite{liao2013dmfsgd,liu2015network,zheng2014investigating}, given partially collected latency measurements between some nodes that are possibly contaminated by noise, the goal is to recover the underlying low-rank latency matrix, which is present due to network path and function correlations.


Matrix factorization is a popular approach for low-rank matrix estimation, in which the underlying matrix $M^*\in \mathbb R^{m\times n}$ is assumed to be $M^* = XY^{\sf T}$, with $X\in \mathbb R^{m\times k}$ and $Y\in \mathbb R^{n\times k}$, such that the rank of $M^*$ is enforced to $k$. The goal is to find $\hat M$ that minimizes the aggregate loss of the estimation $\mathcal A(\hat M)$ on all observed samples $b_i$, $i=1,\ldots,p$.
Matrix factorization problems, although being nonconvex, can be solved efficiently at a large scale by several standard optimization methods such as alternating minimization and stochastic gradient descent.
As a result, matrix factorization has gained enormous success in real-world recommender systems, e.g., Netflix Prize competition \cite{koren2009matrix}, and large-scale network latency estimation, e.g., DMFSGD \cite{liao2013dmfsgd}, due to its scalability, low computation cost per iteration, and the ease of distributed implementation.
In contrast, another approach to matrix estimation and completion, namely nuclear-norm minimization \cite{candes2010power,candes2010matrix} based on SVT \cite{cai2010singular} or proximal gradient methods \cite{ma2011fixed}, is relatively less scalable to problems of huge sizes due to high computational cost per iteration \cite{sun2015guaranteed}.
Recently, a few studies \cite{sun2015guaranteed,jain2013low,zhao2015nonconvex} have also theoretically shown that many optimization algorithms converge to the global optimality of the matrix factorization formulation, and can recover the underlying true low-rank matrix under certain conditions.

Nevertheless, a common limitation of almost all existing studies on matrix estimation is that they have ignored the fact that observations in practice could be highly skewed and do not follow symmetric normal distributions in many applications.
For example, latencies to web services over the Internet are highly skewed, in that most measurements are within hundreds of milliseconds while a small portion of outliers could be over several seconds due to network congestion or temporary service unavailability \cite{zheng2014investigating,liu2015network}.
In a video recommender system based on implicit feedback (e.g., user viewing history), the watching time is also highly skewed, in the sense that a user may watch most videos for a short period of time and only finish a few videos that he or she truly likes \cite{hu2008collaborative}.

In other words, the majority of existing matrix factorization methods are based on least squares and attempt to produce a low-rank matrix $\hat M$ such that $\mathcal A(\hat M)$ estimates the conditional means of observations.
However, in the presence of extreme and skewed data, this may incur large biases and may not fulfill practical requirements. For example, in web service latency estimation, we want to find the \emph{most probable} latency between each client-service pair instead of its conditional mean that is biased towards large outliers. Alternatively, one may be interested in finding the tail latencies and exclude the services with long latency tails from being recommended to a client. Similarly, in recommender systems based on implicit feedback, predicting the conditional mean watching time of each user on a video is meaningless due to the skewness of watching times. Instead, we may want to find out the most likely time length that the user might spend on the video, and based the recommendation on that.
For asymmetric, skewed and heavy-tailed data that are prevalent in the real world, new matrix factorization techniques need to be developed beyond symmetric least squares, in order to achieve robustness to outliers and to better interpret the central tendency or dispersion of observations.

In this paper, we propose the concept of \emph{expectile matrix factorization (EMF)} by replacing the symmetric least squares loss function in conventional matrix factorization with a loss function similar to those used in expectile regression \cite{newey1987asymmetric}.
Our scheme is different from weighted matrix factorization \cite{singh2008unified}, in that we not only assign different weights to different residuals, but assign each weight \emph{conditioned on whether the residual is positive or negative}.
Intuitively speaking, our expectile matrix factorization problem aims to produce a low-rank matrix $\hat M$ such that $\mathcal A(\hat M)$ can estimate any $\omega$th conditional expectiles of the observations, not only enhancing the robustness to outliers, but also offering more sophisticated statistical understanding of observations from a matrix beyond mean statistics. 

We make multiple contributions in this paper. \emph{First}, we propose an efficient algorithm based on alternating minimization and quadratic programming to solve expectile matrix factorization, which has low complexity similar to that of alternating least squares in conventional matrix factorization. \emph{Second}, we theoretically prove that under certain conditions, expectile matrix factorization retains the desirable properties that without noise, it achieves the global optimality and exactly recovers the true underlying low-rank matrices. This result generalizes the prior result \cite{zhao2015nonconvex} regarding the optimality of alternating minimization for matrix estimation under the symmetric least squares loss (corresponding to $\omega = 0.5$ in EMF) to a general class of ``asymmetric least squares'' loss functions for any $\omega\in (0,1)$. 
The results are obtained by adapting a powerful tool we have developed on the theoretical properties of weighted matrix factorization involving varying weights across iterations.
\emph{Third}, for data generated from a low-rank matrix contaminated by skewed noise, we show that our schemes can achieve better approximation to the original low-rank matrix than conventional matrix factorization based on least squares. \emph{Finally}, we also performed extensive evaluation based on a real-world dataset containing web service response times between 339 clients and 5825 web services distributed worldwide. We show that the proposed EMF saliently outperforms the state-of-the-art matrix factorization scheme based on least squares in terms of web service latency recovery from only 5-10\% of samples.




\textbf{Notation}: Without specification, any vector $v=(v_1,\ldots,v_p)^{\sf T} \in \real{p}$ is a column vector. We denote its $l_p$ norm as $\norm{v}_p = \left(\sum_{j}v_j^p\right)^{1/p}$. For a matrix $A \in \real{m \times n}$, we denote $A_{ij}$ as its $(i,j)$-entry. We denote the singular values of $A$ as $\sigma_1(A) \geq \sigma_2(A) \geq \ldots \geq \sigma_k(A)$, where $k = \rank(A)$. Sometimes we also denote $\sigma_{\max}(A)$ as its maximum singular value and $\sigma_{\min}(A)$ as its minimum singular value. We denote $\norm{A}_F = \sqrt{\sum_{j}\sigma_j^2}$ as its Frobenius norm and $\norm{A}_2 = \sigma_{\max}(A)$ as its spectral norm.
For any two matrices $A, B \in \real{m \times n}$, we denote their inner product $\innerprod{A, B} = \tr(A^{\sf T}B) = \sum_{i,j}A_{ij}B_{ij}$.
For a bivariate function $f(x, y)$, we denote the partial gradient w.r.t. $x$ as $\nabla_x f(x, y)$ and that w.r.t. $y$ as $\nabla_y f(x, y)$. 

\section{Expectile Matrix Factorization}
\label{sec:prelim}

Given a linear mapping $\mathcal{A}: \real{m \times n} \to \real{p}$, we can get $p$ observations of an $m \times n$ matrix $M^* \in \real{m\times n}$.
In particular, we can decompose the linear mapping $\mathcal{A}$ into $p$ inner products, i.e., $\innerprod{A_i, M^*}$ for $i=1,\ldots,p$, with $A_i \in \real{m \times n}$.
Denote the $p$ observations by a column vector $b = (b_1,\ldots,b_p)^{\sf T} \in \real{p}$, where $b_i$ is the observation of $\innerprod{A_i, M^*}$ and may contain independent random noise.
The matrix estimation problem is to recover the underlying true matrix $M^*$ from observations $b$, assuming that $M^*$ has a low rank.


Matrix factorization assumes that the matrix $M^*$ has a rank no more than $k$, and can be factorized into two tall matrices $X\in \mathbb R^{m\times k}$ and $Y\in \mathbb R^{n\times k}$ with $k \ll \{m, n, p\}$.
Specifically, it estimates $M^*$ by solving the following non-convex optimization problem:
\begin{equation*}
    \mathop{\min}_{X \in \real{m \times k}, Y \in \real{n \times k}}
    \sum_{i=1}^{p} \mathcal{L} (b_i, \innerprod{A_i, M})
\quad \text{s.t.}\quad M = XY^{\sf T},
\end{equation*}
where $\mathcal{L}(\cdot, \cdot)$ is a loss function. We denote the optimal solution to the problem above by $\hat{M}$. 

The most common loss function used in matrix factorization is the squared loss $(b_i - \innerprod{A_i, XY^{\sf T}})^2$, with which the problem is to minimize the mean squared error (MSE):
\begin{equation}
    \mathop{\min}_{X \in \real{m \times k}, Y \in \real{n \times k}}
    \sum_{i=1}^{p} \frac{1}{2}(b_i - \innerprod{A_i, XY^{\sf T}})^2.
    \label{prob:mse}
\end{equation}
Just like linear regression based on least squares, \eqref{prob:mse} actually aims to produce an $\hat M$ which estimates the conditional mean of $M^*$ given partial observations. For symmetric Gaussian noise, the conditional mean is the most efficient estimator.
However, for skewed or heavy-tailed noise, the conditional mean can be far away from the central area where elements of the true $M^*$ are distributed. In these cases, we need to develop new techniques to better characterize the central tendency, dispersion and tail behavior of observations, beyond mean statistics.


Quantile regression \cite{koenker1978regression} is a type of regression analysis originated in statistics and econometrics and is more robust against outliers especially in heavy-tailed response measurements. However, in quantile regression, we need to minimize a non-smooth check loss function which is more computationally involving.

Similar to quantile regression, expectile regression \cite{newey1987asymmetric} is also a regression technique that achieves robustness against outliers, while in the meantime is more computationally efficient than quantile regression by adopting a smooth loss function. In particular, suppose samples $\{(x_i,y_i),i=1,\ldots,n\}$ are generated from a linear model $y_i = x_i^{\sf T}\beta^* +\varepsilon_i$, where $x_i = (1,x_{i1},\ldots,x_{ip})^{\sf T}\in \mathbb R^{p+1}$ are predictors and $y_i\in\mathbb R$ is the response variable.
The expectile regression estimates $\beta^*$ by solving
{\small \[ \underset{\beta}{\text{minimize}}\quad\sum_{i=1}^n\rho_{\omega}^{[2]}(y_i - x_i^{\sf T}\beta), \]}
where for a chosen constant $\omega \in (0,1)$,  $\rho_\omega^{[2]}(\cdot)$
is the ``asymmetric least squares'' loss function given by
\begin{equation*}
\rho_\omega^{[2]}(t) := t^2\cdot | \omega- \mathbbm{1}(t < 0) |,
\end{equation*}
where $\mathbbm{1}(t<0)$ is the indicator function such that it equals to $1$ if $t<0$ and $0$ otherwise.

Fig.~\ref{fig:exp_loss_func} illustrates the shape of $\rho_\omega^{[2]}(
\cdot)$. When $\omega<0.5$, we can see that the cost of a positive residual is lower than that of a negative residual, thus encouraging a smaller estimate $\hat y_i$ for the response variable, and vice versa when $\omega>0.5$. This fact implies that when the response variable $y_i$ are not Gaussian but highly skewed, we can choose an $\omega$ to push $\hat y_i$ to its most probable area (i.e., the mode or median) while being robust to outliers, as shown in Fig.~\ref{fig:exp_intuit}.



We now extend expectile regression to the case of matrix estimation. Formally, define $r_i := b_i - \innerprod{A_i, XY^{\sf T}}$ as the residual for $b_i$. Then, in loss minimization, we weight each squared residual $r_i^2$ by either $\omega$ or $1-\omega$, conditioned on whether it is positive or negative. Therefore, we formulate \emph{expectile matrix factorization} (EMF) as the following problem:
{\small
\begin{equation}
    \mathop{\min}_{X \in \real{m \times k}, Y \in \real{n \times k}}
    F(X, Y) := \sum_{i=1}^{p} \rho_\omega^{[2]}(b_i-\innerprod{A_i, XY^{\sf T}} ).
    \label{prob:expectile}
\end{equation}
}
Apparently, the MSE-based approach \eqref{prob:mse} is a special case of problem \eqref{prob:expectile} by setting $\omega = 0.5$, which places equal weights on both positive and negative residuals.

\begin{figure}
    \centering
    \subfigure[$\rho_\omega^{[2]}(\cdot)$]{
\begin{tikzpicture}[scale=0.6]
    \draw [very thick, ->] (-2,0) -- (2,0);
    \draw [very thick, ->] (0, -0.5) -- (0,3);
    \draw [cyan, dashed, very thick, domain=0:2] plot(\x, {0.1*\x*\x}) node[right] {\small $\omega=0.1$};
    \draw [cyan, dashed, very thick, domain=-2:0] plot(\x, {0.9*\x*\x});
    \draw [blue, ultra thick, domain=-2:2] plot(\x, {0.5*\x*\x}) node[right] {\small $\omega=0.5$};
    \node [below left] at (0,0) {0};
    \node [below right] at (2,0) {\small $x$};
    \node [left] at (0, 3) {\small $y$};
\end{tikzpicture}

        \label{fig:exp_loss_func}
    }
    \hspace{-3mm}
    \subfigure[Expectile Regression]{
    \begin{tikzpicture}[scale=0.6,domain=.001:10,samples=200,thick]
    \draw [dashed, color=gray] (1.05,0) -- (1.05, 2.5);
    \draw [dashed, color=gray] (2.0,0) -- (2.0, 2.5);
    \draw [dashed, color=gray] (3.42,0) -- (3.42, 2.5);
   	\node[above] at (1.05, 2.5) {\footnotesize $\omega=0.1$};
    \node[below] at (2.0, 0) {\footnotesize  $\omega=0.5$};
    \node[above] at (3.42, 2.5) {\footnotesize $\omega=0.9$};

	\fill[color=cyan, yscale=10, xscale=0.5, domain=0.001:2.1] (0,0) -- plot (\x,{exp(ln(\x/2)*3/2-ln(\x)-\x/2-ln(sqrt(pi)/2))}) -- (2.1, 0);
	\fill[color=magenta, yscale=10, xscale=0.5, domain=6.84:10] (6.84,0) -- plot (\x,{exp(ln(\x/2)*3/2-ln(\x)-\x/2-ln(sqrt(pi)/2))}) -- (10, 0);
    \draw[very thick, <->] (0, 3) -- (0, 0) -- (6, 0);
    \draw[color=black, yscale=10, xscale=0.5, ultra thick] (0,0) -- plot (\x,{exp(ln(\x/2)*3/2-ln(\x)-\x/2-ln(sqrt(pi)/2))});

	\node[left, rotate=90] at (-0.5, 3) {\footnotesize Density of $y_i$};
  \end{tikzpicture}
        \label{fig:exp_intuit}
    }
    \caption{(a) The asymmetric least squares loss function, placing different weights on positive residuals and negative residuals.
(b) For a skewed $\chi^2_3$ distribution, expectile regression with $\omega = 0.1$ generates an estimate closer to the mode than the conditional mean ($\omega = 0.5$) does due to the long tail.}
\label{fig:intuition}
\end{figure}
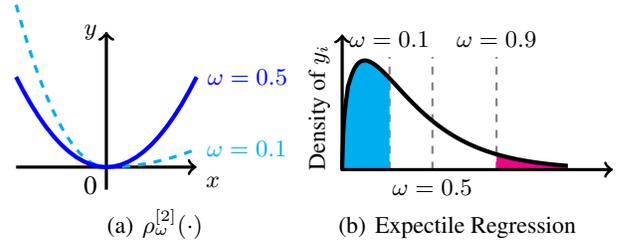


Note that expectile matrix factorization proposed above is different from weighted MSE \cite{singh2008unified}, where a different yet fixed (predefined) weight is assigned to different residuals. In expectile matrix factorization, each weight is either $\omega$ or $1-\omega$, depending on whether the residual of the estimate is positive or negative, i.e., we do not know the assignments of weights before solving the optimization problem.
In other words, problem \eqref{prob:expectile} estimates an $\hat M$ such that each $\innerprod{A_i, \hat{M}}$ estimates the $\omega$th \emph{conditional expectile} \cite{newey1987asymmetric} of $b_i$.
In the meantime, expectiles are based on second-order moments and thus it is feasible to solve EMF efficiently, which we show in the next section.

Just like expectile regression, the main attraction of expectile matrix factorization goes beyond robustness to outliers.
Being able to estimate any $\omega$th expectile of observations, EMF can characterize different measures of central tendency and statistical dispersion, and is useful to obtain a more comprehensive understanding of data distribution. For example, if we are interested in the tail behavior, we could set $\omega = 0.9$ and if we are interested in the conditional median in a highly skewed dataset, we could set $\omega<0.5$.

\section{Algorithm and Theoretical Results}
\label{sec:algorithm}

We propose an efficient algorithm to solve expectile matrix factorization
via a combined use of alternating minimization and quadratic programming, as shown in Algorithm~\ref{alg:bcd_expectile}, with complexity similar to that of alternating least squares in conventional matrix factorization. 
To better approach potential optimal solutions, we first sum up all measurement matrices $A_i$ weighted by $b_i$, and perform Singular Value Decomposition (SVD) to get top $k$ singular values.

\begin{algorithm}[htbp]
\caption{Alternating minimization for expectile matrix factorization. In this algorithm, we use $\bar{X}$ to highlight that $\bar{X}$ is orthonormal.}
\label{alg:bcd_expectile}
\begin{algorithmic}[1]
\State \textbf{Input}: observations $b = (b_1,\ldots,b_p)^{\sf T} \in \real{p}$, measurement matrices $A_i \in \real{m \times n}$, $i=1,\ldots, p$.
\State \textbf{Parameter}: Maximum number of iterations $T$
\State $(\bar{X}^{(0)}, D^{(0)}, \bar{Y}^{(0)})=\texttt{SVD}_k(\sum_{i=1}^{p}b_i A_i)$
\Comment{Singular Value Decomposition to get top $k$ singular values}
\For {$t=0$ to $T-1$}
    \State $Y^{(t+0.5)} \gets \arg\min_{Y} F(\bar{X}^{(t)}, Y)$
    \State $\bar{Y}^{(t+1)} \gets \texttt{QR}(Y^{(t+0.5)})$ \Comment{QR decomposition}
    \State $X^{(t+0.5)} \gets \arg\min_{X} F(X, \bar{Y}^{(t+1)})$
    \State $\bar{X}^{(t+1)} \gets \texttt{QR}(X^{(t+0.5)})$

\EndFor
\State \textbf{Output}: $M^{(T)} \gets X^{(T-0.5)}\bar{Y}^{(T)\sf T}$
\end{algorithmic}
\end{algorithm}

The QR decompositions in Step 6 and Step 8 are \emph{not necessary} and are only included here to simplify the presentation of theoretical analysis.
QR decomposition ensures the orthonormal property: given an orthonormal matrix $X$ (or $Y$), the objective function $F(X, Y)$ is strongly convex and smooth with respect to $Y$ (or $X$), as shown in the appendix.
However, it has been proved \cite{jain2013low} that when $\omega=0.5$,  alternating minimization with and without QR decomposition are equivalent. The same conclusion also holds for all $\omega$. Therefore, in performance evaluation, we do not have to and did not apply QR decomposition.

The subproblems in Step 5 and Step 7 can be solved efficiently with standard quadratic program (QP) solvers after some reformulation. We now illustrate such equivalence to QP for Step 5, which minimizes $F(\bar{X}, Y)$ given $\bar{X}$. Let $r_i^+:=\max(r_i, 0)$ denote the positive part of residual $r_i$, and $r_i^- := -\min(r_i, 0)$ denote the negative part of $r_i$. We have $r_i = r_i^+ - r_i^-$, and the asymmetric least squares loss can be rewritten as
\begin{equation*}
    \rho_{\omega}^{[2]}(r_i) = \omega (r_i^+)^2 + (1-\omega)(r_i^-)^2.
\end{equation*}
Given $\bar{X}$, we have
\[ \mathcal{A}(\bar{X}Y^{\sf T}) = \{ \innerprod{A_i, \bar{X}Y^{\sf T}} \}_{i=1}^p = \{ \innerprod{A_i^{\sf T}\bar{X}, Y} \}_{i=1}^p.\]

Let $r^+ = (r_1^+,\ldots,r_p^+)^{\sf T}$ and $r^- = (r_1^-,\ldots,r_p^-)^{\sf T}$. For simplicity, let $\mathcal{A}_1(Y):=\mathcal{A}(\bar{X}Y^{\sf T})$. Then, minimizing $F(\bar{X}, Y)$ given $\bar{X}$ in Step 5 is equivalent to the following QP:
{\small
\begin{equation}
    \begin{split}
        \underset{Y \in \real{n \times k} , r^+, r^-\in \real{p}_+}{\text{min}} \quad &
        \omega \norm{r^+}_2^2 + (1-\omega)\norm{r^-}_2^2 \\
        \text{s.t.} \quad & r^+ - r^- = b - \mathcal{A}_1(Y).
    \end{split}
    \label{prob:sub_expectile}
\end{equation}
}
Similarly, Step 7 can be reformulated as a QP as well.

Steps 5 and 7 can be solved even more efficiently in the matrix completion case, which aims at recovering an incomplete low-rank matrix from a few observed entries and is a special case of the matrix estimation problem under discussion, where each $b_i$ is simply an observation of a matrix element (possibly with noise). In matrix completion, we can decompose the above QP in Steps 5 and 7 by updating each row of $X$ (or $Y$), whose time complexity in practice is similar to conventional alternating least squares, e.g., \cite{koren2009matrix}, which also solve QPs.



\begin{figure*}[!htb]
  \centering
  \subfigure[Sampling rate $R=0.05$]{
	\includegraphics[width=2.7in]{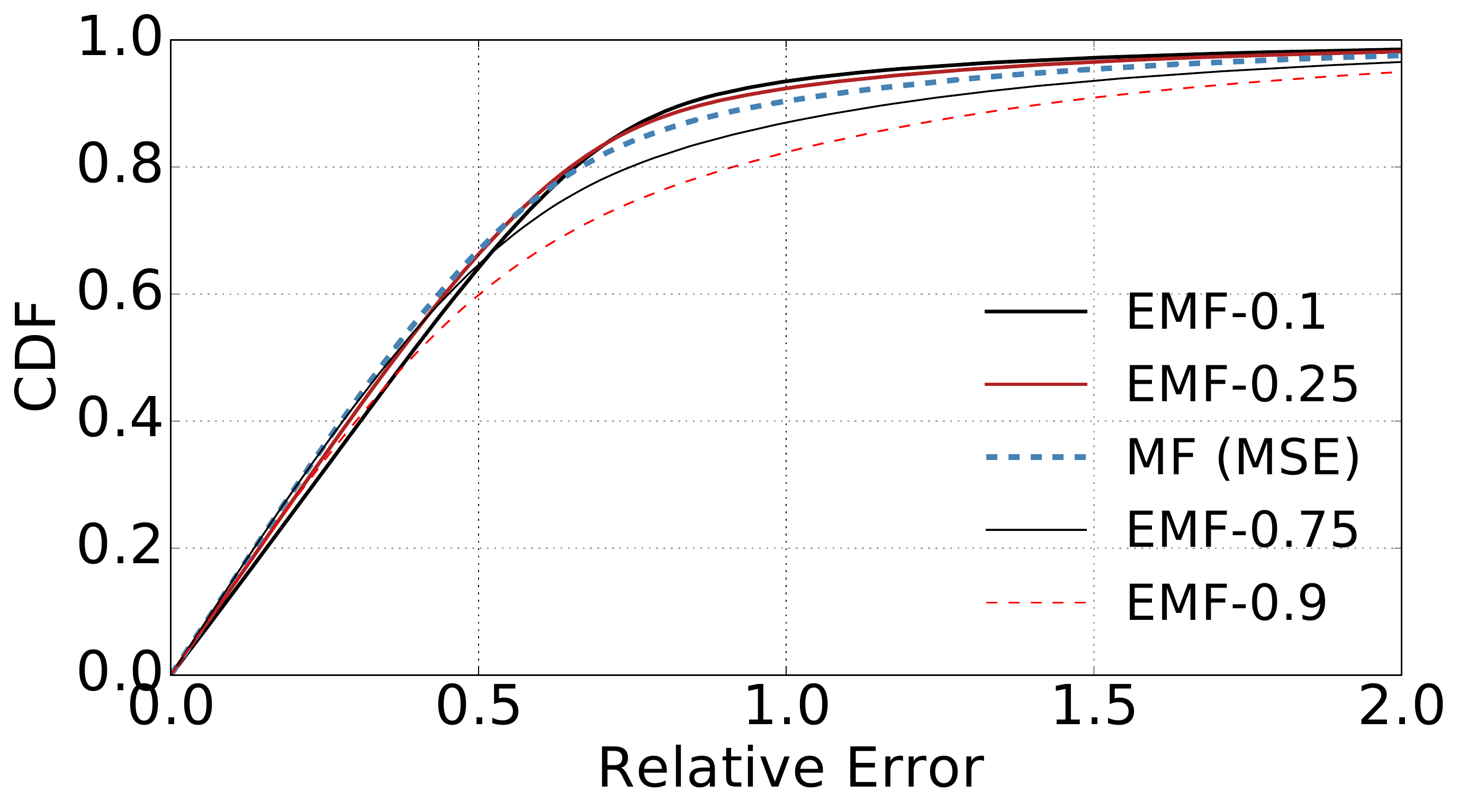}
	\label{fig:exp05}
  }
  \subfigure[Sampling rate $R=0.1$]{
    \includegraphics[width=2.7in]{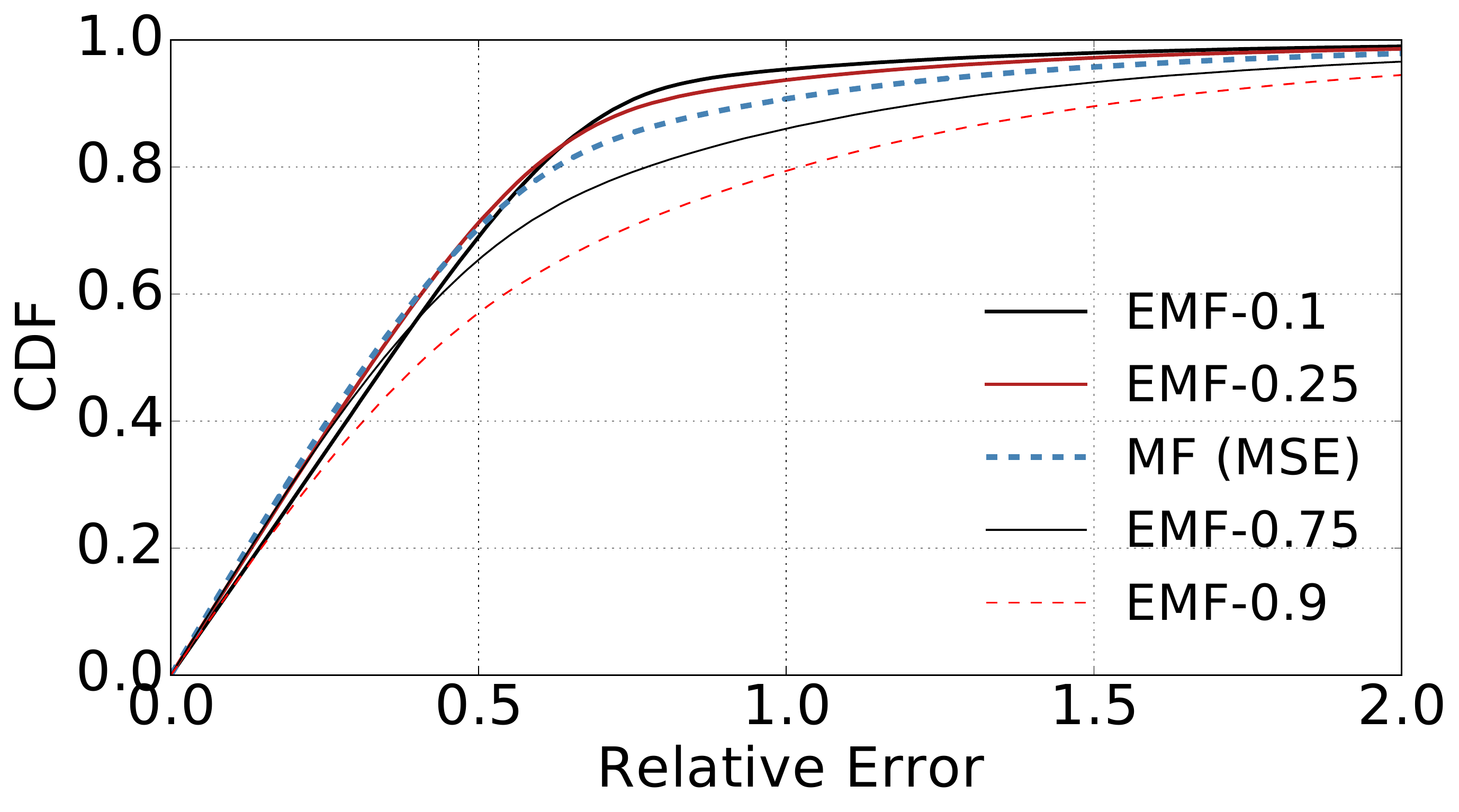}
    \label{fig:exp10}
  }
  \caption{CDF of relative errors via expectile matrix factorization on synthetic $1000\times 1000$ matrices with skewed noise.}  
  \label{fig:compareMC}
\end{figure*}

We now show that the proposed algorithm for expectile matrix factorization retains the optimality for any $\omega \in (0,1)$ when observations are noiseless, i.e., the produced $M^{(T)}$ will eventually approach the true low-rank matrix $M^*$ to be recovered. We generalize the recent result  \cite{zhao2015nonconvex} of the optimality of alternating minimization for matrix estimation under the symmetric least squares loss function (corresponding to $\omega = 0.5$ in EMF) to a general class of ``asymmetric least squares'' loss functions with any $\omega\in (0,1)$.

We assume that the linear mapping $\mathcal{A}$ satisfies the well-known $2k$-RIP condition \cite{jain2013low}:
\begin{asmp}[$2k$-RIP]
There exists a constant $\delta_{2k} \in (0,1)$ such that for any matrix $M$ with rank at most $2k$, the following property holds:
    \[ (1-\delta_{2k}) \lVert {M}\rVert_F^2
    \leq \lVert \mathcal{A}({M}) \rVert_2^2
    \leq (1+\delta_{2k}) \lVert {M}\rVert_F^2.\]
\end{asmp}
A linear mapping $\mathcal{A}$ satisfying the RIP condition can be obtained in various ways. For example, if each entry of $A_i$ is independently drawn from the sub-Gaussian distribution, then $\mathcal{A}$ satisfies $2k$-RIP property with high probability for $p=\Omega(\delta_{2k}^{-2} kn\log n)$ \cite{jain2013low}.

Clearly, Algorithm~\ref{alg:bcd_expectile} involves minimizing a weighted sum of squared losses in the form of 
$$\mbox{$\mathcal{F}(X, Y)=\sum_{i=1}^p w_i (b_i - \innerprod{A_i, XY^{\sf T}})^2$},$$
although the weight $w_i$  depends on the sign of residual $r_i$ and may vary in each iteration.
We show that if the weights $w_i$ are confined within a closed interval $[w_-, w_+]$ with constants $w_-, w_+ > 0$, then the alternating minimization algorithm for the weighted sum of squared losses will converge to an optimal point. Without loss of generality, we can assume that $w_- \leq 1/2 \leq w_+$ and $w_- + w_+ = 1$ by weight normalization.

First, we show the geometric convergence of alternating minimization for weighted matrix factorization, if all weights belong to $[w_-, w_+]$ in each iteration:
\begin{thm}
    Assume that the linear mapping $\mathcal{A}(\cdot)$ satisfies $2k$-RIP condition with $\delta_{2k} \leq C_1/k\cdot w_-^2/w_+^2$ for some small constant $C_1$, and assume that the singular values of $M^*$ are bounded in the range of $[\Sigma_{\min}, \Sigma_{\max}]$, where $\Sigma_{\min}$ and $\Sigma_{\max}$ are constants and do not scale with the matrix size. Suppose the weights in $\mathcal{F}(X, Y)$ are bounded by two positive finite constants, i.e., $w_i \in [w_-, w_+]$ with $0 < w_- \leq 1/2 \leq w_+ < 1$ and $w_- + w_+ = 1$.
    Then, given any desired precision $\varepsilon$, there exists a constant $C_2$ such that by applying alternating minimization to $\mathcal{F}(X, Y)$, the solution $M^{(T)}$ satisfies $\norm{M^{(T)} - M^*}_F \leq \varepsilon$ for all $T \geq O(\log (C_2/\varepsilon) + \log(w_-/w_+) )$.
    \label{thm:weighted}
\end{thm}
The detailed proof of the above theorem is quite involving and is included in the supplemental material. Theorem~\ref{thm:weighted} implies that the weighted matrix factorization can geometrically converge to a global optimum. Note that the negative term $\log(w_- / w_+)$ does not imply that weighted matrix factorization converges faster, since the value of $C_2$ for two $w$'s may differ. In fact, due to the lower RIP constant $\delta_{2k}$, the convergence rate in the case of $w_- \neq w_+$ is usually slower than that in the case of $w_- = w_+$.

In Algorithm~\ref{alg:bcd_expectile} for expectile matrix factorization, the weight $w_i$ in each iteration for residual $r_i$ is $\omega$ if $r_i \geq 0$, and is $1-\omega$ otherwise. Although $w_i$ is changing across iterations, we can choose $w_- = \min(\omega, 1-\omega)$ and $w_+=\max(\omega, 1-\omega)$, both satisfying the assumptions in Theorem~\ref{thm:weighted}, to bound all $w_i$. Then we can derive the following main result directly from Theorem~\ref{thm:weighted}.

\begin{thm}[Optimality of Algorithm~\ref{alg:bcd_expectile}]
    Suppose $\omega \leq 1/2$. Assume that the linear mapping $\mathcal{A}(\cdot)$ satisfies $2k$-RIP condition with $\delta_{2k} \leq C_3/k \cdot (1-\omega)^2/\omega^2$ for some small constant $C_3$, and assume that the singular values of $M^*$ are bounded in the range of $[\Sigma_{\min}, \Sigma_{\max}]$, where $\Sigma_{\min}$ and $\Sigma_{\max}$ are constants and do not scale with the matrix size.
    Then, given any desired precision $\varepsilon$, there exists a constant $C_4$ such that Algorithm~\ref{alg:bcd_expectile} satisfies $\norm{M^{(T)} - M^*}_F \leq \varepsilon$ for all $T \geq O(\log (C_4/\varepsilon) + \log(\omega/(1-\omega)) )$. If $\omega > 1/2$, we can get the same result by substituting $\omega$ with $1-\omega$.
    \label{thm:expectile_exact}
\end{thm}

Additionally, 
the number of observations needed for exact recovery is
$p = \Omega\big(\frac{(1-\omega)^2}{\omega^2} k^3 n \log n\big),$
if the entries of $A_i$ are independently drawn from a sub-Gaussian distribution with zero mean and unit variance,
since we require $\delta_{2k} \leq C/k \cdot (1-\omega)^2/\omega^2$. This also matches the sampling complexity of conventional matrix factorization \cite{jain2013low}. 

\begin{figure}
    \centering
    \subfigure[Response times]{
        \includegraphics[width=1.5in]{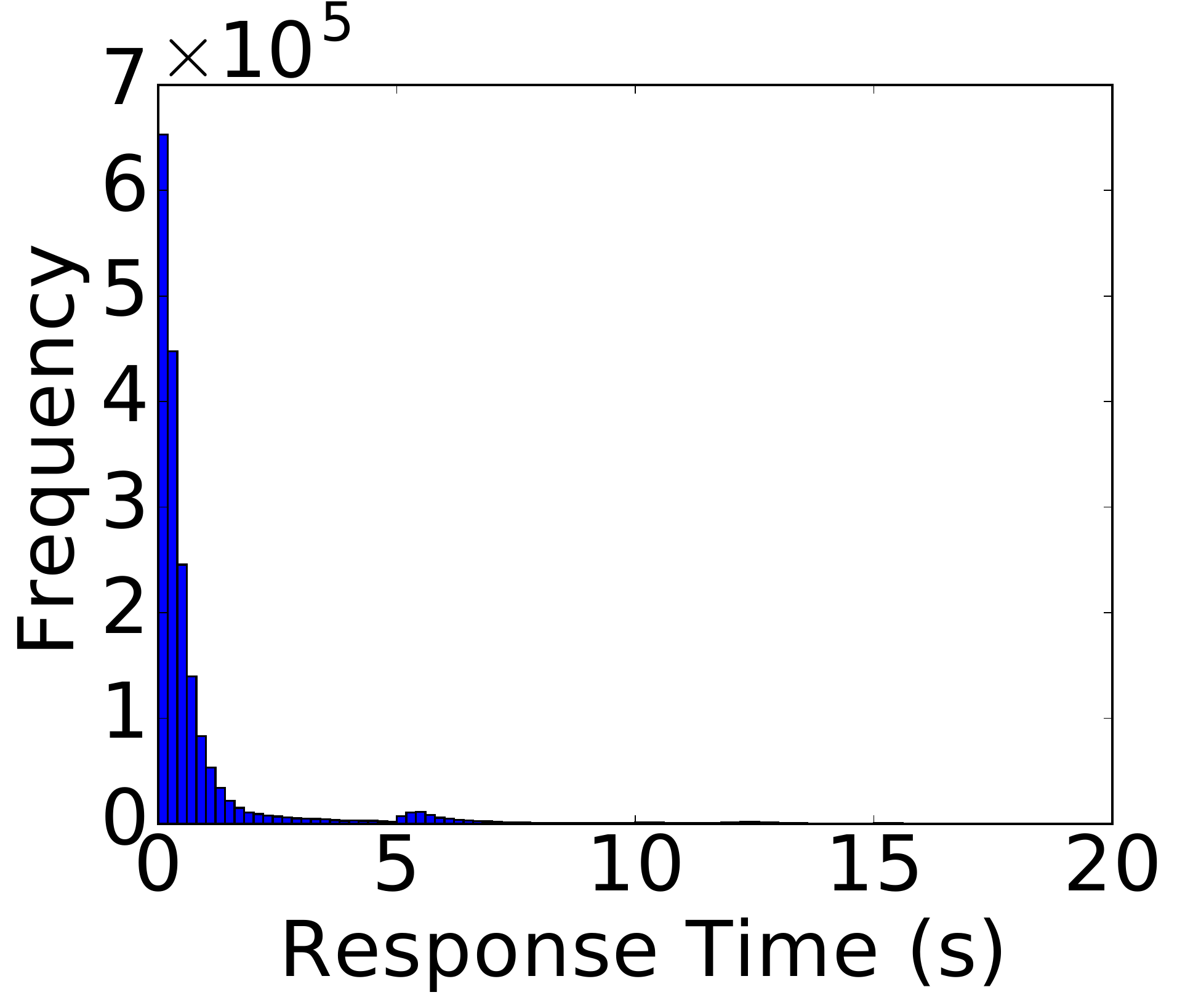}
        \label{fig:rt_stat}
    }
    \hspace{-2mm}
    \subfigure[Residuals from MSE-MF]{
        \includegraphics[width=1.6in]{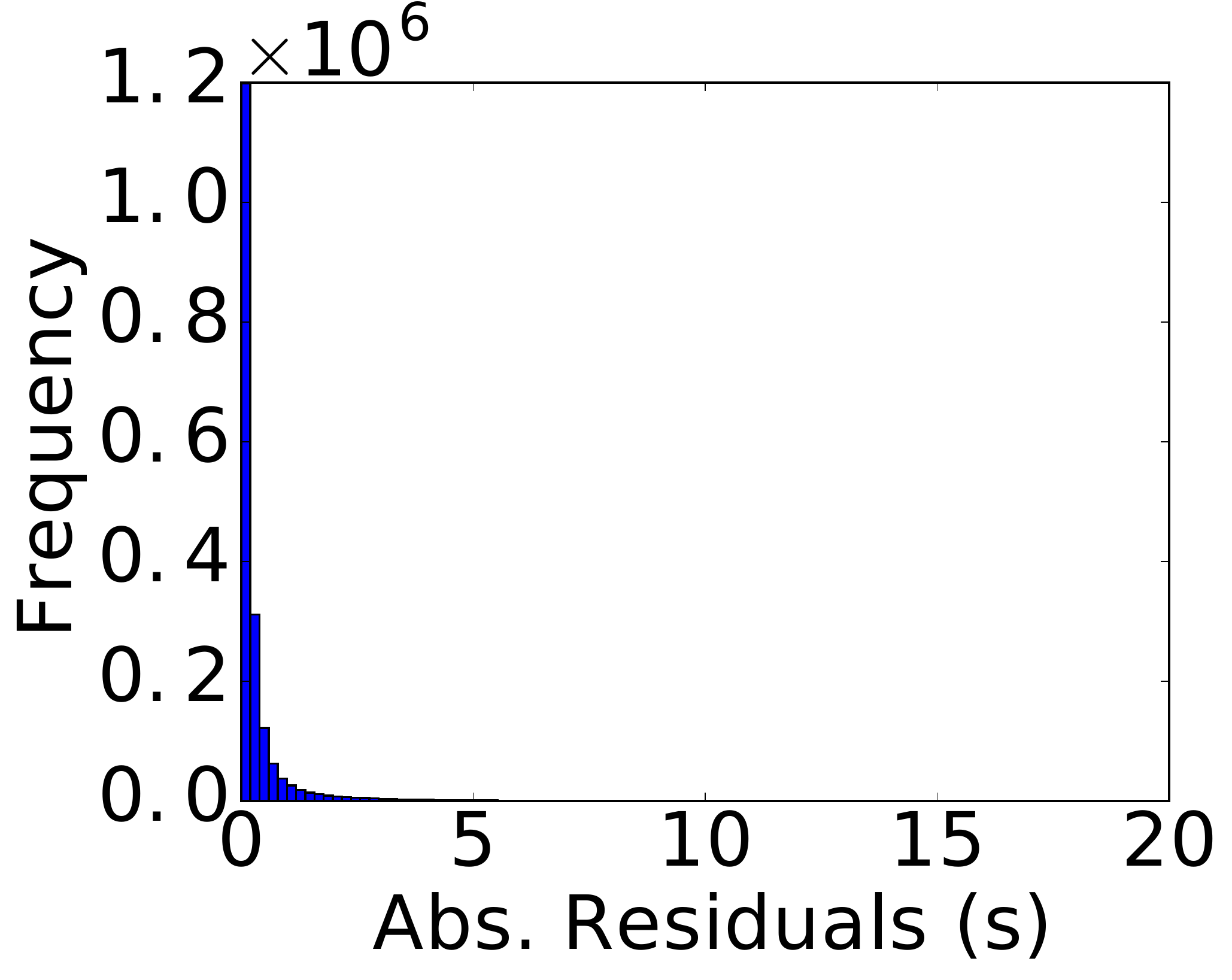}
        \label{fig:rt_residual}
    }
    \caption{Histograms of a) response times between 5825 web services and 339 service users; b) the residuals of estimates from MSE-based matrix factorization applied on the complete matrix.}
\end{figure}

\begin{figure*}[!htb]
    \centering
    \subfigure[Sampling rate $R=0.05$]{
        \includegraphics[width=2.7in]{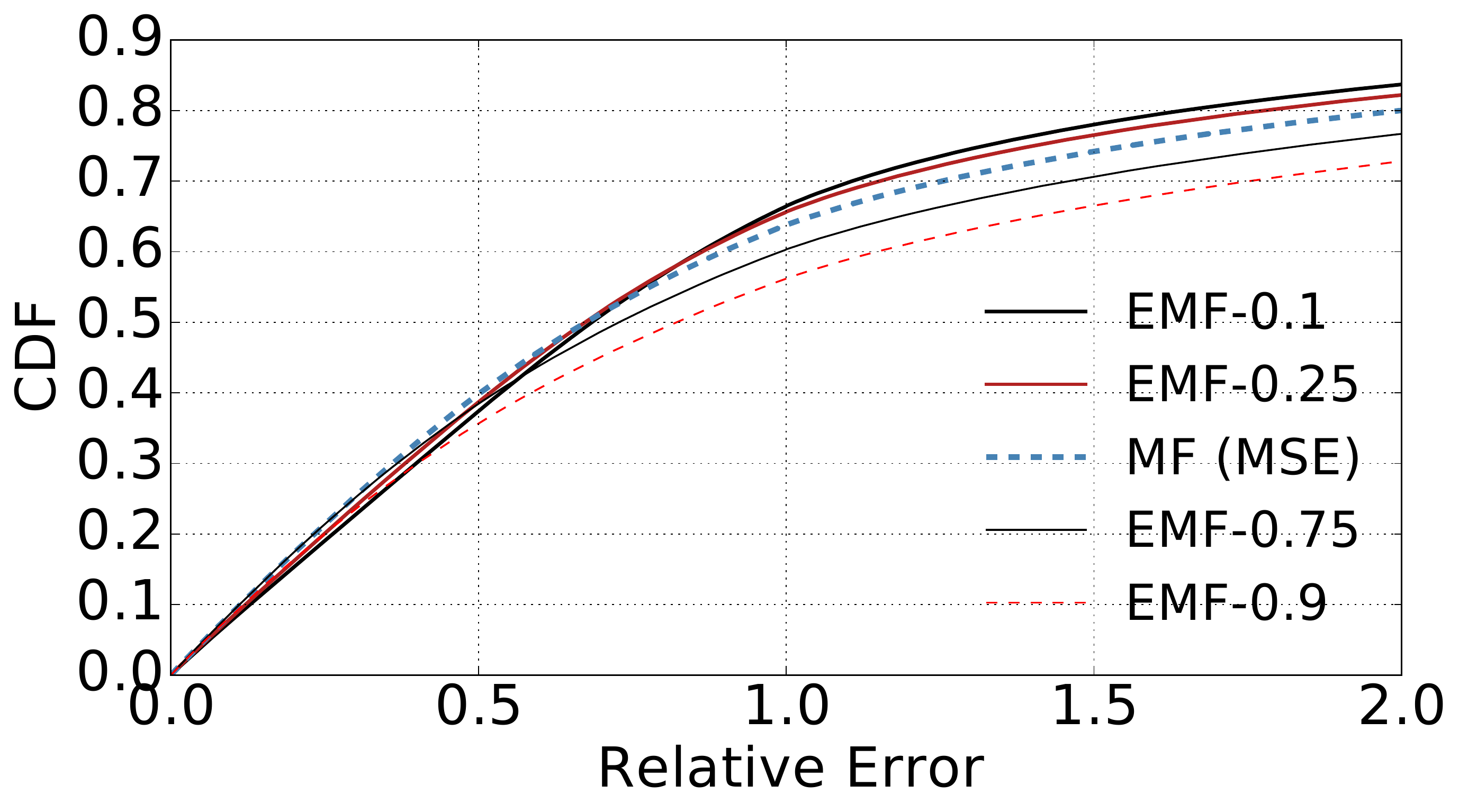}
        \label{fig:rt05}
    }
    \subfigure[Sampling rate $R=0.1$]{
      \includegraphics[width=2.7in]{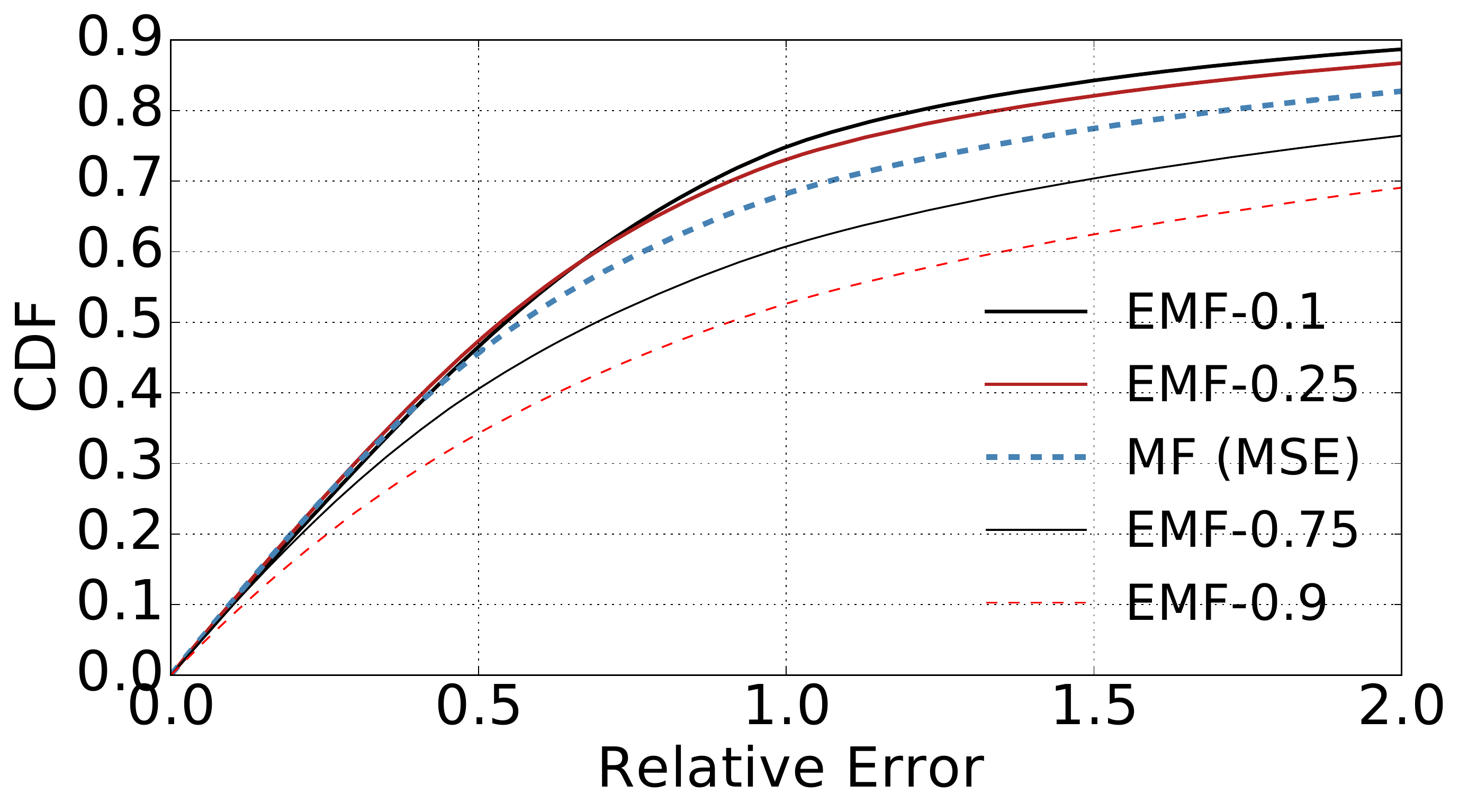}
      \label{fig:rt10}
    }
    \caption{CDF of relative errors via expectile matrix factorization for web service response time estimation under different sampling rates and $\omega$.}
    \label{fig:rt_cdf}
\end{figure*}

\section{Experiments}
\label{sec:simu}

In this section, we evaluate the performance of EMF in comparison to the state-of-the-art MSE-based matrix factorization based on both skewed synthetic data and a real-world dataset containing web service response times between 339 users and 5825 web services collected worldwide \cite{zheng2014investigating}. In both tasks, we aim to estimate a true matrix $M^*$ based on partial observations.
We define the relative error (RE) as $|M^*_{i,j} - \hat{M}_{i,j}| / M^*_{i,j}$ for all the missing entries $(i,j)$. We use RE to evaluate the prediction accuracy of different methods under a certain sampling rate $R$ (the fraction of known entries).


\subsection{Experiments on Skewed Synthetic Data}
We randomly generate a $1000\times 1000$ matrix $M^* = X Y^{\sf T}$ of rank $k=10$, where $X \in \real{m \times k}$ and $Y\in \real{n \times k}$ have independent and uniformly distributed entries in $[0,1]$.
Then, we contaminate $M^*$ by a skewed noise matrix $0.5N$, where $N$ contains independent \emph{Chi-square} entries with 3 degrees of freedom. The 0.5 is to make sure the noise does not dominate. We observe some elements in the contaminated matrix and aim to recover the underlying true low-rank $M^*$ under two sampling rates $R=0.05$ and $R=0.1$, respectively, where $R$ is the fraction of elements observed. The experiment is repeated for 10 times for each $R$.
We plot the CDF of relative errors in terms of recovering the missing elements of $M^*$ in Fig.~\ref{fig:compareMC}. We can see that expectile matrix factorization outperforms the conventional MSE-based algorithm (EMF with $\omega=0.5$) in terms of recovery from skewed noise, with $\omega=0.1$ yielding the best performance, under both $R=0.05$ and $R=0.1$. When more observations are available with $R=0.1$, EMF with $\omega=0.1$ demonstrates more benefit as it is more robust to the heavy-tailed noise in data.



\begin{figure*}[!htb]
  \subfigure[EMF--0.1]{
	\includegraphics[width=1.35in]{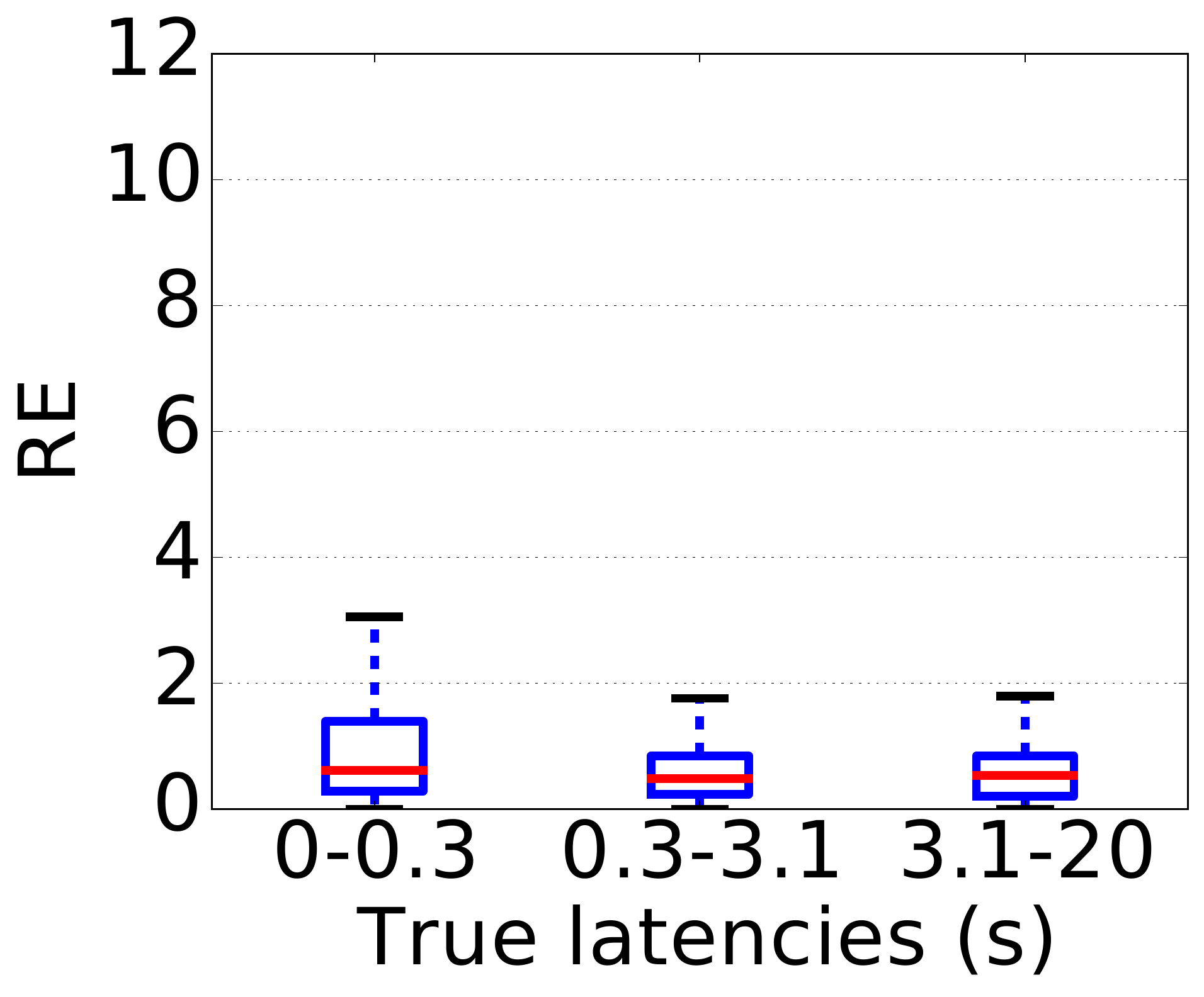}
	\label{fig:boxplot10}
  }
  \hspace{-3mm}
  \subfigure[EMF--0.25]{
    \includegraphics[width=1.35in]{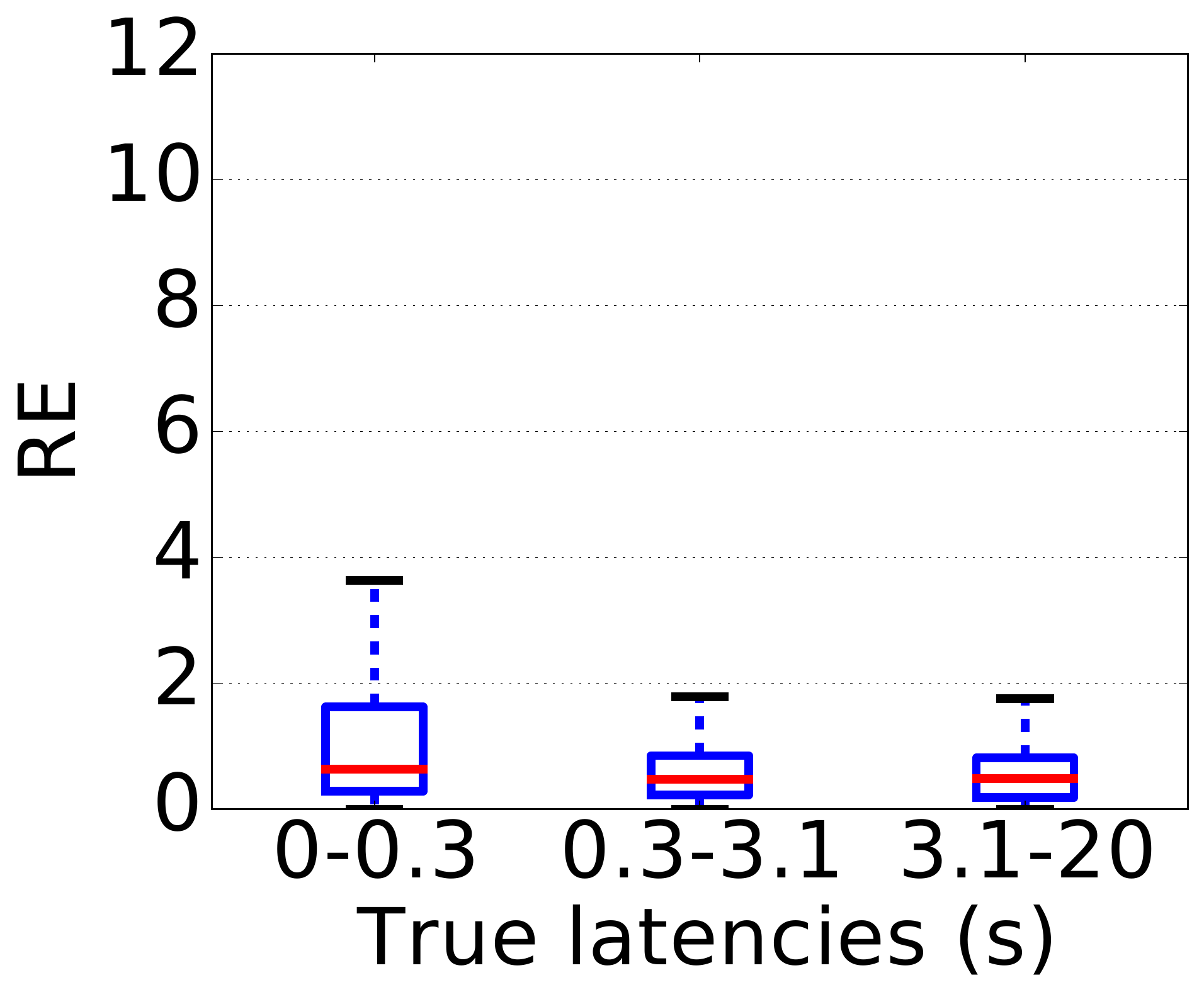}
    \label{fig:boxplot25}
  }
  \hspace{-3mm}
  \subfigure[MF--MSE (EMF--0.5)]{
    \includegraphics[width=1.35in]{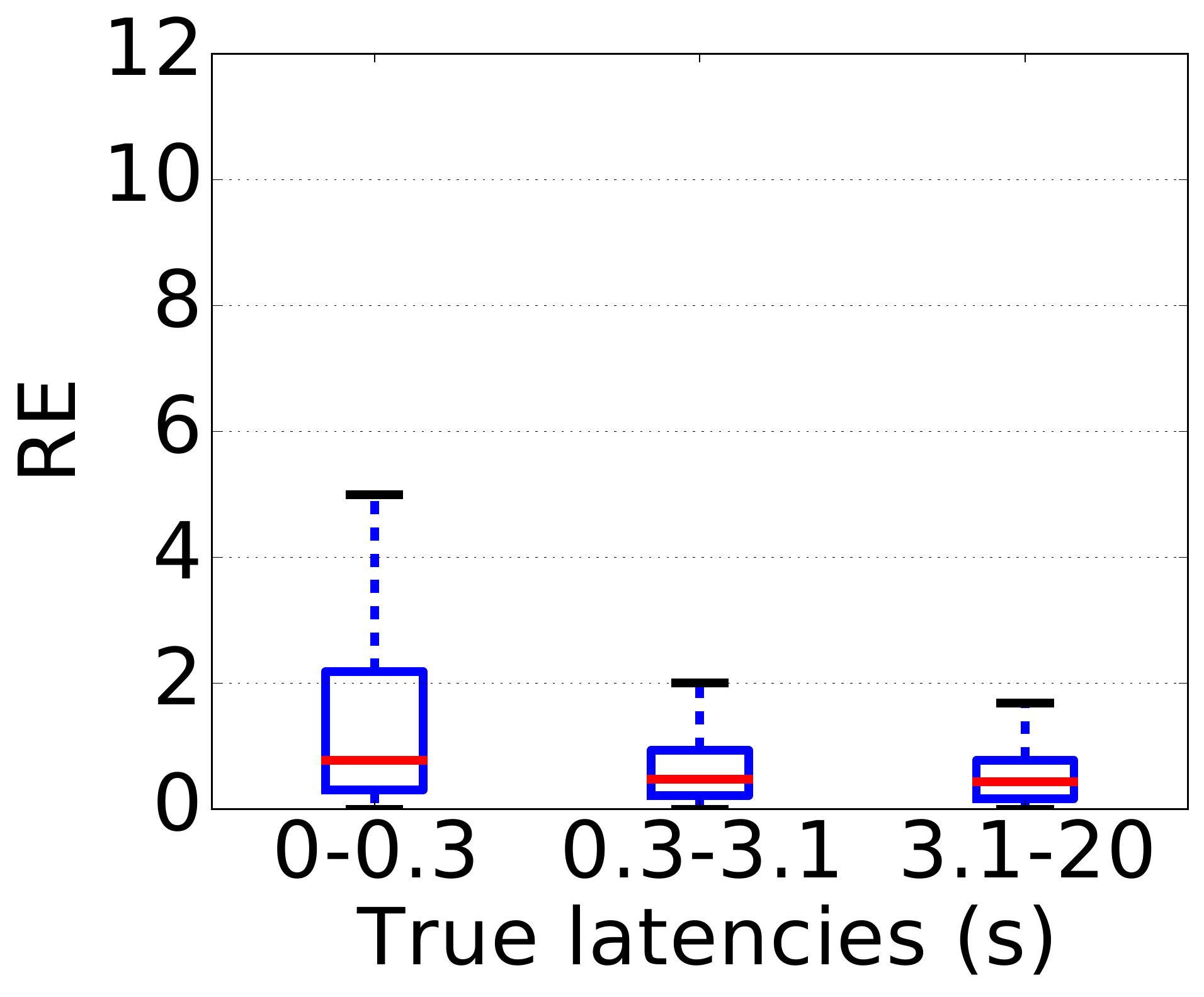}
    \label{fig:boxplot50}
  }
  \hspace{-3mm}
  \subfigure[EMF--0.75]{
    \includegraphics[width=1.35in]{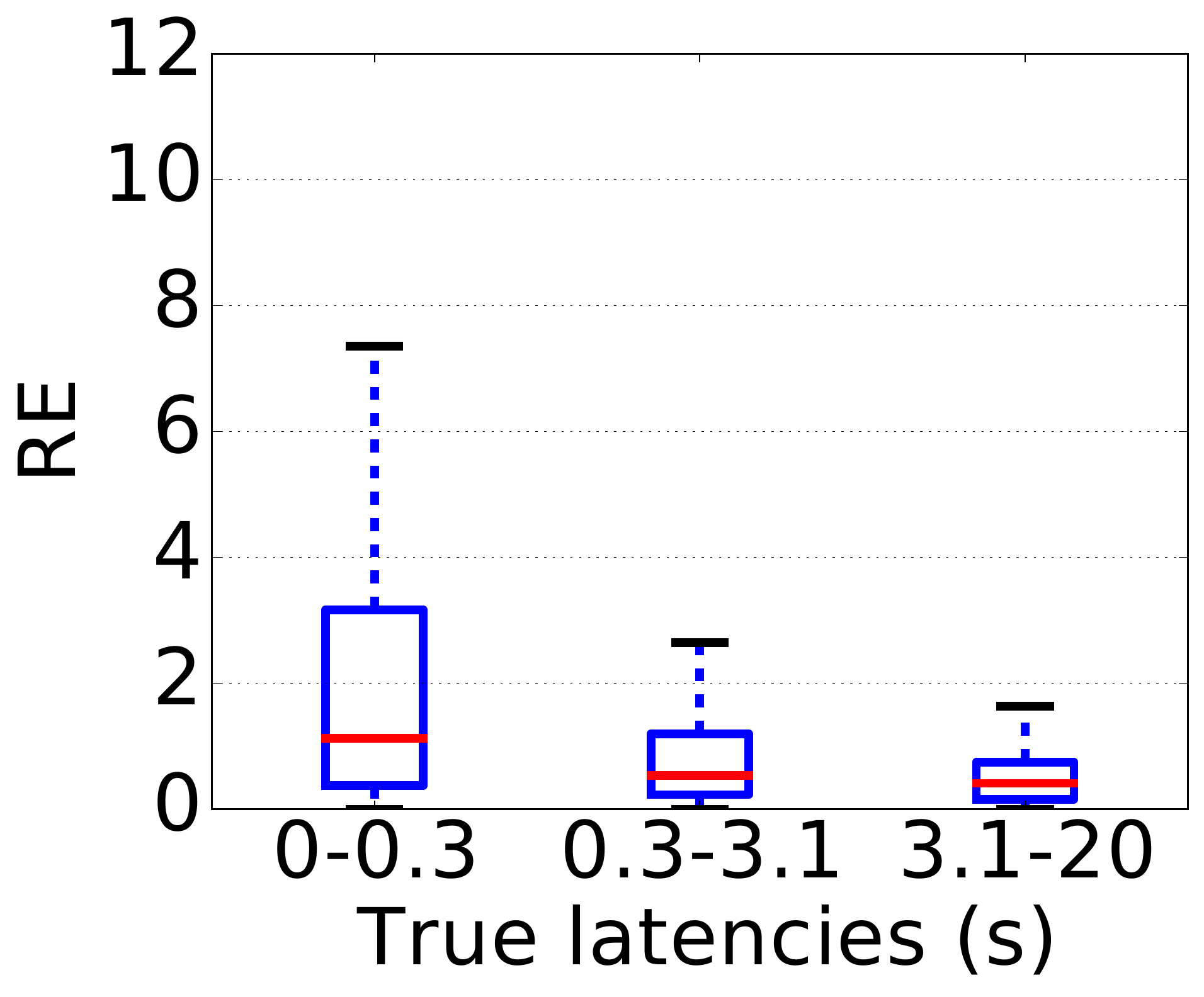}
    \label{fig:boxplot75}
  }
  \hspace{-3mm}
  \subfigure[EMF--0.9]{
    \includegraphics[width=1.35in]{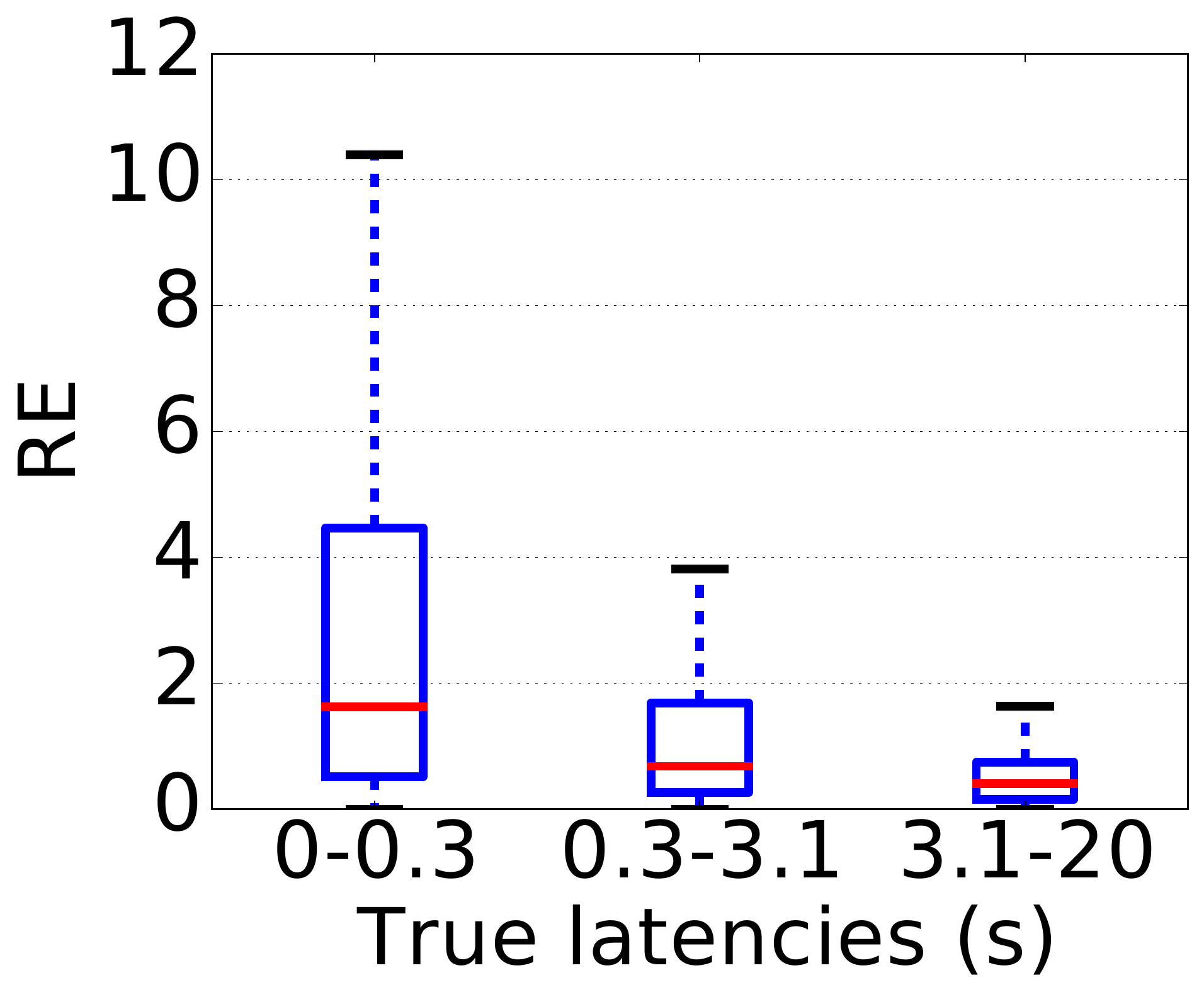}
    \label{fig:boxplot90}
  }
  \caption{Box plots of relative errors for different bins of true latencies in the test sets.}
  \label{fig:boxplotRE}
\end{figure*}

\begin{figure}
    \centering
    \subfigure[Median of REs]{
  	\includegraphics[width=1.65in]{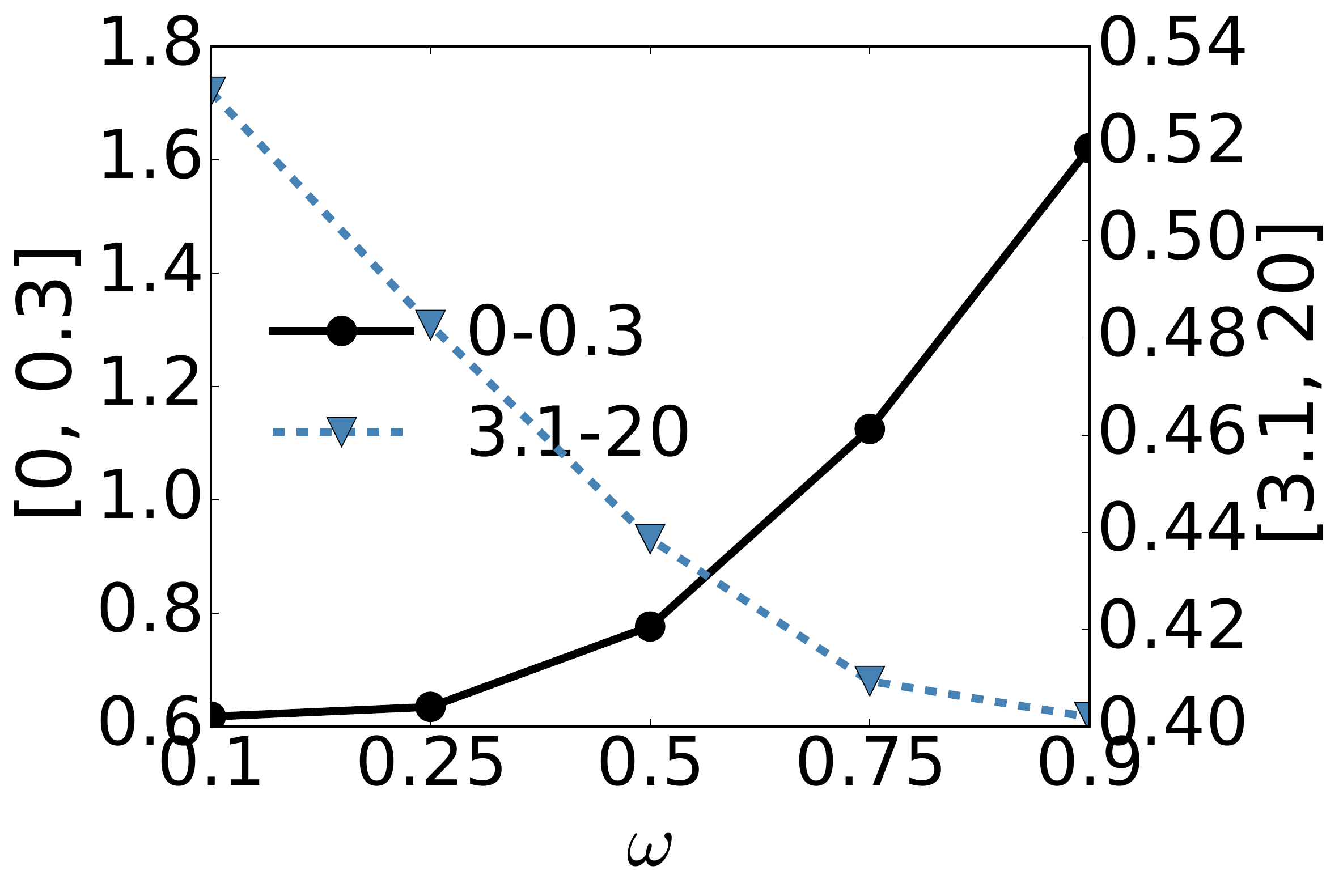}
  	\label{fig:median}
    }
    \hspace{-3mm}
    \subfigure[IQR of REs]{
      \includegraphics[width=1.55in]{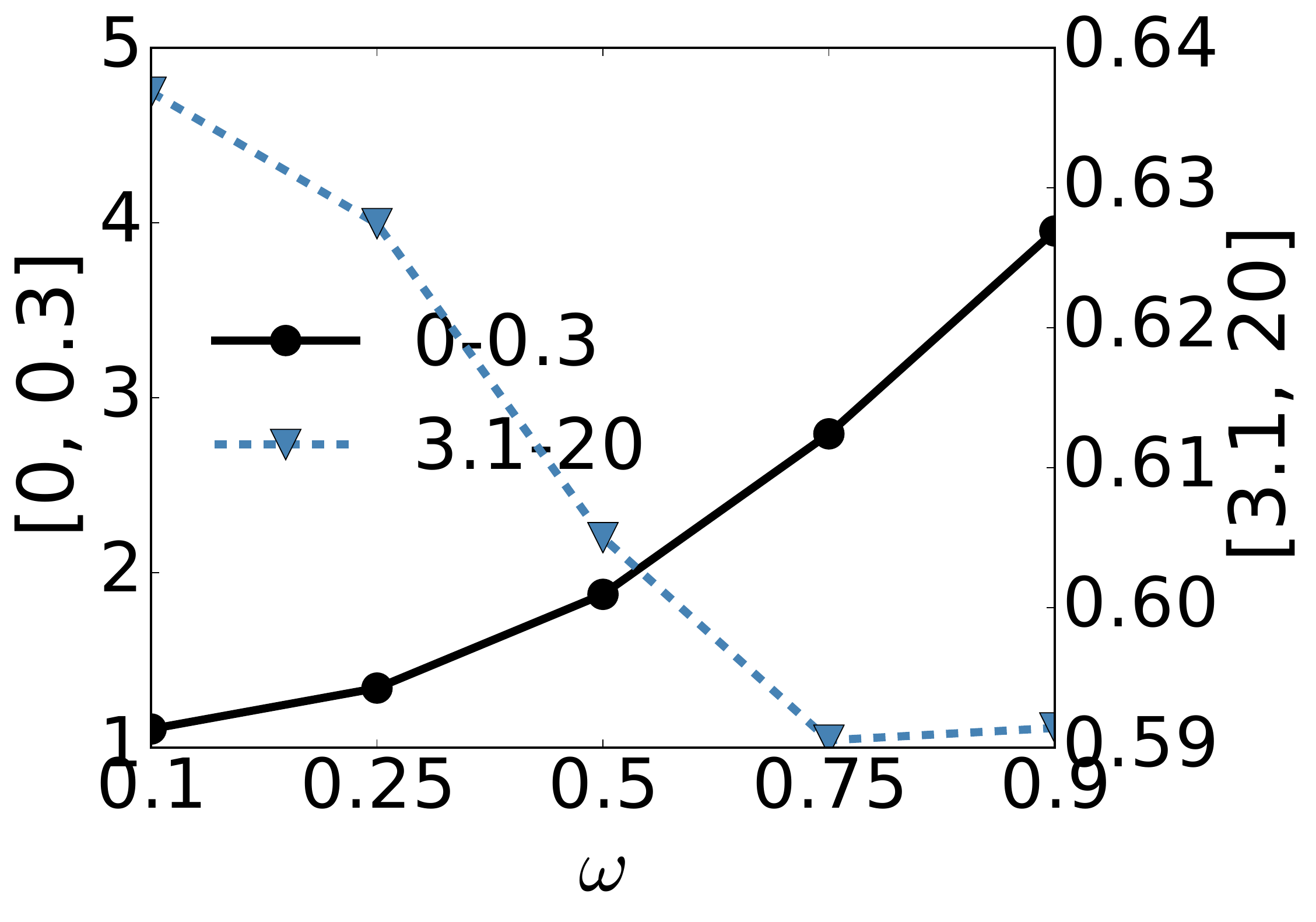}
      \label{fig:iqr}
    }
    \caption{The medians and IQRs of relative errors for different bins as $\omega$ varies.}
    \label{fig:median_iqr}
\end{figure}

\subsection{Experiments on Web Service Latency Estimation}

In these experiments, we aim to recover the web service response times  between 339 users and 5825 web services \cite{zheng2014investigating} distributed worldwide, under different sampling rates.

Fig.~\ref{fig:rt_stat} shows the histogram of all the response times measured between 339 users and 5825 web services. While most entries are less than 1 second, some response times may be as high as 20 seconds due to network delay variations, software glitches and even temporary service outages. 
The mean latency is 0.91 second, whereas the median is only 0.32 second. This implies that the mean is heavily impacted by the few tail values, while the 0.1-th expectile, which is 0.3 second, is closer to the median of the data. Therefore, if we use the conventional MSE-based matrix factorization to recover this skewed data, the result can be far away from the central area, while EMF with $\omega=0.1$ may better explain the central tendency. 

We further performed the MSE-based matrix factorization for the complete response time matrix, which boils down to  singular value decomposition (SVD) and we plot the residual histogram in Fig.~\ref{fig:rt_residual}. In this figure, 90\% of residuals are less than 0.8, while the largest residual can be up to 19.73. Since the residuals are still highly skewed, the conditional means do not serve as good estimates for the most probable data.

In Fig.~\ref{fig:rt_cdf}, we plot the relative errors of recovering missing response times with EMF under different $\omega$. Note that EMF-0.5 is essentially the conventional MSE-based matrix factorization. In Fig.~\ref{fig:rt_cdf}, we can see that EMF-0.1 performs the best under both sampling rates, and EMF-0.9 performs the worst, because the 0.1-th expectile is the closest to the median, while both the mean and 0.9-th expectile are far away from the central area of data distribution.

To take a closer look at the performance EMF on different segments of data, we divide the testing response times into three bins: 0-0.3s containing 47.5\% of all entries, 0.3-3.1s containing 45.4\% of all entries, and 3.1-20s containing only 7.1\% of all entries.
We show the relative errors for testing samples from different bins in box plots in Fig.~\ref{fig:boxplotRE} under different $\omega$.
In addition, in Fig.~\ref{fig:median_iqr}, we plot the median of REs and the interquartile range (IQR, the gap between the upper and lower quartiles) of REs  when $R=0.1$, as $\omega$ varies for the lower latency bin and the higher latency bin, respectively.

We can observe that EMF with a lower $\omega$ achieves higher accuracy in the lower range 0-0.3s, while EMF with a higher $\omega$ can predict better in the higher end 3.1-20s. This observation conforms to the intuition illustrated in Fig.~\ref{fig:intuition}: an $\omega<0.5$ penalizes negative residuals, pushing the estimates to be more accurate on the lower end, where most data are centered around. From Fig.~\ref{fig:median} and Fig.~\ref{fig:iqr}, we can see that EMF-0.1 predicts the best for the lower range, while EMF-0.9 performs the best for the higher range. However, since most data are distributed in the lower range, EMF-0.1 is better at predicting the central tendency and achieves the best overall accuracy.

\section{Concluding Remarks}
\label{sec:conclude}

In this paper, we propose the expectile matrix factorization approach (EMF) which introduces the ``asymmetric least squares'' loss function of expectile regression analysis originated in statistics and econometrics into matrix factorization for robust matrix estimation. Existing matrix factorization techniques aim at minimizing the mean squared error and essentially estimate the conditional means of matrix entries. 
In contrast, the proposed EMF can yield the $\omega$th conditional expectile estimates of matrix entries for any $\omega \in (0,1)$, accommodating the conventional matrix factorization as a special case of $\omega = 0.5$.
We propose an efficient alternating minimization algorithm to solve EMF and theoretically prove its convergence to the global optimality in the noiseless case.
Through evaluation based on both synthetic data and a dataset containing real-world web service response times, we show that EMF achieves better recovery than conventional matrix factorization when the data is skewed or contaminated by skewed noise. By using a flexible $\omega$, EMF is not only more robust to outliers but can also be tuned to obtain a more comprehensive understanding of data distribution in a matrix, depending on application requirements.

\bibliographystyle{aaai}
\bibliography{main}

\onecolumn


\section*{Preliminaries}
\begin{lem}[Lemma B.1 of \cite{jain2013low}]
    Suppose $\mathcal{A}(\cdot)$ satisfies $2k$-RIP. For any $X, U \in \real{m \times k}$ and $Y, V \in \real{n \times k}$, we have
    \[ \vert \innerprod{\mathcal{A}(XY^{\sf T}), \mathcal{A}(UV^{\sf T})} - \innerprod{X^{\sf T}U, Y^{\sf T}V} \vert \leq 3\delta_{2k}\norm{XY^{\sf T}}_F \cdot \norm{UV^{\sf T}}_F \]
    \label{lem:dec1-0}
\end{lem}

\begin{lem}[Lemma 2.1 of \cite{jain2010guaranteed}]
    Let $b = \mathcal{A}(M^*) + \varepsilon$, where $M^*$ is a matrix with the rank of $k$, $\mathcal{A}$ is the linear mapping operator satisfies $2k$-RIP with constant $\delta_{2k} < 1/3$, and $\varepsilon$ is a bounded error vector. Let $M^{(t+1)}$ be the $t+1$-th step iteration of SVP, then we have
    \[\norm{\mathcal{A}(M^{(t+1)}) - b}_2^2 \leq \norm{ \mathcal{A}(M^*) - b}_2^2 + 2\delta_{2k} \norm{\mathcal{A}(M^{(t)}) - b}_2^2. \]
    \label{lem:rip_svp}
\end{lem}

\begin{lem}[Lemma 4.5 of \cite{zhao2015nonconvex}]
    Suppose that $Y^{(t+0.5)}$ in Alg.~\ref{alg:bcd_expectile} satisfies $\norm{Y^{(t+0.5)} - V^{(t)}}_F \leq \sigma_k/4$. Then, there exists a factorization of matrix $M^* = U^{(t+1)}\bar{V}^{(t+1)\sf T}$ such that $V^{(t+1)} \in \real{n \times k}$ is an orthonomal matrix, and satisfies \[\norm{\bar{Y}^{(t+1)} - \bar{V}^{(t+1)}}_F \leq 2/\sigma_k \cdot \norm{Y^{(t+0.5)} - V^{(t)}}_F. \]
    \label{lem:qr_dec}
\end{lem}

\section*{Proofs}

\subsection{Roadmap}

The first step is to prove strongly convexity and smoothness of $\mathcal{F}(X, Y)$ if one variable is fixed by a orthonormal matrix as follows:
\begin{lem}
    Suppose that $\delta_{2k}$ and $\bar{X}^{(t)}$ satisfy
    \begin{equation}
        \delta_{2k} \leq \frac{\sqrt{2}w_-^2(1-\delta_{2k})^2\sigma_k}{24\xi w_+ k(1+\delta_{2k})\sigma_1}.
    \end{equation}
    and
    \begin{equation}
        \norm{\bar{X}^{(t)} - \bar{U}^{(t)}}_F \leq \frac{w_-(1-\delta_{2k})\sigma_k}{2\xi w_+(1+\delta_{2k})\sigma_1}
    \end{equation}
    Then we have:
    \[\norm{Y^{(t+0.5)} - V^{(t)}}_F \leq \frac{\sigma_k}{2\xi} \norm{\bar{X}^{(t)}-\bar{U}^{(t)}}_F. \]
    \label{lem:dec1}
\end{lem}

Clearly, Algorithm~\ref{alg:bcd_expectile} involves minimizing a weighted sum of squared losses in the form of $\mathcal{F}(X, Y)=\sum_{i=1}^p w_i (b_i - \innerprod{A_i, XY^{\sf T}})^2$, although the weight $w_i$  depends on the sign of residual $r_i$ and may vary in each iteration.
We show that the if the weights $w_i$ are confined in a closed interval $[w_-, w_+]$ with constants $w_-, w_+ > 0$, then the alternating minimization algorithm for the weighted sum of squared losses will converge to the optimal point. Without loss of generality, we can assume that $w_- \leq 1/2 \leq w_+$ and $w_- + w_+ = 1$ by weight normalization. For notation simplicity, we denote a finite positive constant $\xi > 1$ throughout this paper.

\begin{lem}
    Suppose the linear operator $\mathcal{A}(\cdot)$ satisfies $2k$-RIP with parameter $\delta_{2k}$. For any orthonormal matrix $\bar{X} \in \mathbb{R}^{m \times k}$, the function $\mathcal{F}(\bar{X}, Y)$ with bounded weights is strongly convex and smooth. In particular, if any weight $w_i$ in $\mathcal{F}(\bar{X}, Y)$ belongs to $[w_-, w_+]$, the value of
    \[
        \mathcal{F}(\bar{X}, Y')-\mathcal{F}(\bar{X}, Y) - \innerprod{\nabla_Y \mathcal{F}(\bar{X}, Y), Y'-Y}
    \]
    is bounded by
    \[
        [w_-(1-\delta_{2k})\norm{Y'-Y}_F^2, w_+(1+\delta_{2k})\norm{Y'-Y}_F^2]
    \]
    for all $Y, Y'$.
    \label{lem:convex}
\end{lem}

Lemma~\ref{lem:convex} shows that $\mathcal{F}(X, Y)$ can be block-wise strongly convex and smooth if the weights $w_i$ belongs to $[w_-, w_+]$. In the following, we use $U$ and $V$ to denote the optimal factorization of $M^* = UV^{\sf T}$. Note that $U$ and $V$ are unique up to orthogonal transformations. The following lemma shows that by taking the block-wise minimum, the distance between the newly updated variable $Y^{(t+0.5)}$ and its ``nearby'' $V^{(t)}$ is upper bounded by the distance between $X^{(t)}$ and its corresponding neighbor $U^{(t)}$.

\begin{lem}
    Suppose that $\delta_{2k}$ satisfies
    \[\delta_{2k} \leq \frac{w_-^2(1-\delta_{2k})^2 \sigma_{k}^4}{48\xi^2 k w_+^2(1+\delta_{2k})^2 \sigma_1^4}.\]
    We have $\norm{\bar{Y}^{(t+1)} - \bar{V}^{(t+1)}}_F \leq \frac{1}{\xi}\norm{\bar{X}^{(t)} - \bar{U}^{(t)}}_F$.
    \label{lem:descent}
\end{lem}

With the above three lemmas, we can prove Theorem~\ref{thm:weighted} by iteratively upper bounded the distance $\norm{\bar{Y}^{(t)} - \bar{V}^{(t)}}_F$ as well as $\norm{\bar{X}^{(t)} - \bar{U}^{(t)}}_F$.

\subsection{Proof of Lemma~\ref{lem:convex}}
\label{proof:convex}

Now we begin to prove these lemmas. Note that a similar technique has also been used by \cite{zhao2015nonconvex}. Since we should fix $X^{(t)}$ or $Y^{(t)}$ as orthonormal matrices, we perform a QR decomposition after getting the minimum. The following lemma shows the distance between $\bar{Y}^{(t+1)}$ and its ``nearby'' $\bar{V}^{(t+1)}$ is still under control. Due to the page limit, we leave all the proofs in the supplemental material.

\begin{proof}
    Since $\mathcal{F}(\bar{X}, Y)$ is a quadratic function, we have
    \begin{eqnarray*}
        \mathcal{F}(\bar{X}, Y') &=& \mathcal{F}(\bar{X}, Y) + \innerprod{\nabla_Y \mathcal{F}(\bar{X}, Y), Y'-Y}\\
        && +\frac{1}{2}(\myvec(Y') - \myvec(Y))^{\sf T} \nabla_Y^2 \mathcal{F}(\bar{X}, Y) (\myvec(Y') - \myvec(Y)),
    \end{eqnarray*}
    and it suffices to bound the singular values of the Hessian matrix $S_{\omega}:= \nabla_Y^2 \mathcal{F}(\bar{X}, Y)$ so that
    \begin{eqnarray*}
        \mathcal{F}(\bar{X}, Y') - \mathcal{F}(\bar{X}, Y) - \innerprod{\nabla_Y \mathcal{F}(\bar{X}, Y), Y'-Y} &\leq& \frac{\sigma_{\max}(S_{\omega})}{2} \norm{Y'-Y}_F^2\\
        \mathcal{F}(\bar{X}, Y') - \mathcal{F}(\bar{X}, Y) - \innerprod{\nabla_Y \mathcal{F}(\bar{X}, Y), Y'-Y} &\geq& \frac{\sigma_{\min}(S_{\omega})}{2} \norm{Y'-Y}_F^2.
    \end{eqnarray*}
    Now we proceed to derive the Hessian matrix $S_{\omega}$. Using the fact $\myvec(AXB) = (B^{\sf T} \otimes A)\myvec(X)$, we can write $S_{\omega}$ as follows:
    \begin{eqnarray*}
        S_{\omega} &=& \sum_{i=1}^{p} 2 w_i\cdot \myvec(A_i^{\sf T} \bar{X}) \myvec^{\sf T}(A_i^{\sf T} \bar{X}) \\
        &=& \sum_{i=1}^{p}2 w_i \cdot (I_k \otimes A_i^{\sf T})\myvec(\bar{X})\myvec^{\sf T}(\bar{X})(I_k \otimes A_i).
    \end{eqnarray*}
    Consider a matrix $Z \in \real{n \times k}$ with $\norm{Z}_F = 1$, and we denote $z = \myvec(Z)$. Then we have
    \begin{eqnarray*}
        z^{\sf T} S_{\omega} z &=& \sum_{i=1}^{p} 2w_i \cdot z^{\sf T}(I_k \otimes A_i^{\sf T})\myvec(\bar{X})\myvec^{\sf T}(\bar{X})(I_k \otimes A_i) \\
        &=& \sum_{i=1}^p 2w_i \cdot \myvec^{\sf T}(A_i Z) \myvec(\bar{X})\myvec^{\sf T}(\bar{X})\myvec(A_i Z) \\
        &=& \sum_{i=1}^p 2w_i \cdot \tr^2(\bar{X}^{\sf T}A_i Z) = \sum_{i=1}^p 2w_i \cdot \tr^2(A_i^{\sf T}\bar{X}Z^{\sf T}).
    \end{eqnarray*}
    From the $2k$-RIP property of $\mathcal{A_i }$, we have
    \begin{eqnarray*}
        z^{\sf T} S_{\omega} z &\leq& \sum_{i=1}^{p} 2 w_+ \tr^2(\bar{X}^{\sf T}A_i Z) \\
        &\leq& 2 w_+ (1+\delta_{2k})\norm{\bar{X}Z^{\sf T}}_F \\
        &=& 2 w_+ (1+\delta_{2k})\norm{Z^{\sf T}}_F = 2w_+(1+\delta_{2k}).
    \end{eqnarray*}
    Similarly, we also have
    \[z^{\sf T} S_{\omega} z \geq 2w_-(1-\delta_{2k}).\]
    Therefore, the maximum singular value $\sigma_{\max}$ is upper bounded by $2w_+(1+\delta_{2k})$ and the minimum singular value $\sigma_{\min}$ is lower bounded by $2w_-(1-\delta_{2k})$, and the Lemma has been proved.
\end{proof}

\subsection{Proof of Lemma~\ref{lem:dec1}}
\label{proof:dec1}
We prove this lemma by introducing a divergence function as follows.
\[
\mathcal{D}(Y^{(t+0.5)},Y^{(t+0.5)},\bar{X}^{(t)}) = \left\langle \nabla_Y \mathcal{F}(\bar{U}^{(t)}, Y^{(t+0.5)}) - \nabla_Y \mathcal{F}(\bar{X}^{(t)}, Y^{(t+0.5)}), \frac{Y^{(t+0.5)} - V^{(t)}}{\norm{Y^{(t+0.5)} - V^{(t)}}_F} \right\rangle .
\]

\begin{lem}
    Under the same condition in Lemma~\ref{lem:dec1}, we have
    \begin{equation}
        \mathcal{D}(Y^{(t+0.5)},Y^{(t+0.5)},\bar{X}^{(t)}) \leq \frac{3(1-\delta_{2k})\sigma_k}{2\xi}\cdot\frac{w_+^2}{w_-} \norm{\bar{X}^{(t)}-\bar{U}^{(t)}}.
    \end{equation}
    \label{lem:dec1-1}
\end{lem}

\begin{proof}[Proof of Lemma~\ref{lem:dec1-1}]
    In this proof we omit the iteration superscript, and $Y$ stands particularly for $Y^{(t+0.5)}$. Since $b_i$ is measured by $\innerprod{A_i, \bar{U}V^{\sf T}}$, we have
    \begin{equation*}
        \mathcal{F}(\bar{X}, Y) = \sum_{i=1}^p w_i (\innerprod{A_i, \bar{X}Y^{\sf T}} - \innerprod{A_i, \bar{U}V^{\sf T}})^2.
    \end{equation*}
    By taking the partial derivatives on $Y$ we have
    \begin{eqnarray*}
        \nabla_Y \mathcal{F}(\bar{X}, Y) &=& \sum_{i=1}^p 2w_i(\innerprod{A_i, \bar{X}Y^{\sf T}} - \innerprod{A_i, \bar{U}V^{\sf T}}) A_i^{\sf T} X \\
        &=& \sum_{i=1}^p 2w_i(\innerprod{A_i^{\sf T}\bar{X}, Y} - \innerprod{A_i^{\sf T}\bar{U}, V}) A_i^{\sf T} X
    \end{eqnarray*}
    Let $x:=\myvec(\bar{X})$, $y := \myvec(Y)$, $u:=\myvec(\bar{U})$, and $v := \myvec(V)$. Since $Y$ minimizes $\mathcal{F}(\bar{X}, \hat{Y})$, we have
    \begin{eqnarray*}
        \myvec(\nabla_Y \mathcal{F}(\bar{X}, Y))
        &=& \sum_{i=1}^p 2w_i(\innerprod{A_i^{\sf T}\bar{X}, Y} - \innerprod{A_i^{\sf T}\bar{U}, V}) A_i^{\sf T} x \\
        &=& \sum_{i=1}^p 2w_i ((\myvec(A_i^{\sf T}\bar{X})\cdot \innerprod{A_i^{\sf T}\bar{X}, Y} - \myvec(A_i^{\sf T}\bar{X})\cdot \innerprod{A_i^{\sf T}\bar{X}, Y})) \\
        &=& \sum_{i=1}^p 2w_i ((I_k \otimes A_i^{\sf T})xx^{\sf T}(I_k \otimes A_i)y  - (I_k \otimes A_i^{\sf T})xu^{\sf T}(I_k \otimes A_i)v )
    \end{eqnarray*}
    We denote
    \[
        S_\omega = \sum_{i=1}^{p}2 w_i \cdot (I_k \otimes A_i^{\sf T})xx^{\sf T}(I_k \otimes A_i),
    \]
    and
    \[
        J_\omega = \sum_{i=1}^{p}2 w_i \cdot (I_k \otimes A_i^{\sf T})xu^{\sf T}(I_k \otimes A_i),
    \]
    So the equation becomes $S_{\omega}y - J_{\omega}v = 0$ and since $S_{\omega}$ is invertible we have $y = (S_{\omega})^{-1} J_{\omega}v$. Meanwhile, we denote
    \[
        G_\omega = \sum_{i=1}^{p}2 w_i \cdot (I_k \otimes A_i^{\sf T})uu^{\sf T}(I_k \otimes A_i)
    \]
    as the Hessian matrix of $\nabla_Y^2 \mathcal{F}(\bar{U}, Y)$. Then, the partial gradient $\nabla_Y \mathcal{F}(\bar{U}, Y)$ can be written as
    \begin{eqnarray*}
        \myvec(\nabla_Y \mathcal{F}(\bar{U}, Y))
        &=& \sum_{i=1}^p 2w_i(\innerprod{A_i^{\sf T}\bar{U}, Y} - \innerprod{A_i^{\sf T}\bar{U}, V}) (I_k \otimes A_i^{\sf T})u \\
        &=& \sum_{i=1}^p 2w_i ((I_k \otimes A_i^{\sf T})uu^{\sf T}(I_k \otimes A_i)y  - (I_k \otimes A_i^{\sf T})uu^{\sf T}(I_k \otimes A_i)v ) \\
        &=& G_{\omega}(y-v) \\
        &=& G_{\omega}(S_{\omega}^{-1}J_{\omega}-I_{nk})v.
    \end{eqnarray*}
    Since we have $\myvec(\nabla_Y \mathcal{F}(\bar{X}, Y)) = 0$, the divergence $\mathcal{D}=\innerprod{\nabla_Y(\bar{U}, Y), (Y-V)/\norm(Y-V)}_F$. So we need to bound $\nabla_Y \mathcal{F}(\bar{U}, Y)$. Let $K := \bar{X}^{\sf T}\bar{U} \otimes I_n$. To get the estimate of $S_{\omega}^{-1}J_{\omega}-I_{nk}$, we rewrite it as
    \[ S_{\omega}^{-1}J_{\omega}-I_{nk} = K-I_{nk} + S_{\omega}^{-1}(J_{\omega} - S_{\omega}K).\]
    We firstly bound the term $(K-I_{nk})v$. Recall $\myvec(AXB) = (B^{\sf T} \otimes A)\myvec(X)$, we have
    \begin{eqnarray*}
        (K-I_{nk})v &=& ((\bar{X}^{\sf T}\bar{U} - I_k)\otimes I_n)v = \myvec(V(\bar{U}^{\sf T}X - I_k)) \\
        \norm{(K-I_{nk})v}_2 &=& \norm{V(\bar{U}^{\sf T}\bar{X} - I_k)}_F \leq \sigma_1 \norm{\bar{U}^{\sf T}\bar{X}-I_k}_F \\
        &\leq& \sigma_1 \norm{(\bar{X}-\bar{U})^{\sf T}(\bar{X}-\bar{U})}_F \leq \sigma_1 \norm{\bar{X}-\bar{U}}_F^2
    \end{eqnarray*}
    We then bound the term $J_{\omega} - S_{\omega}K$. For any two matrices $Z_1, Z_2\in \real{n\times k}$, we denote $z_1 := \myvec(Z_1)$ and $z_2 := \myvec(Z_2)$. Then we have:
    \begin{eqnarray*}
        && z_1^{\sf T}(S_{\omega}K - J_{\omega})z_2 \\
        &=& \sum_{i=1}^{p} 2w_i z_1^{\sf T}(I_k \otimes A_i^{\sf T})x \{x^{\sf T}(I_k \otimes A_i)(\bar{X}^{\sf T}\bar{U}\otimes I_n)) - u^{\sf T}(I_k \otimes A_i)\}z_2 \\
        &=& \sum_{i=1}^{p} 2w_i \innerprod{Z_1, A_i^{\sf T}\bar{X}}\cdot (x^{\sf T}(\bar{X}^{\sf T}\bar{U} \otimes A_i)z_2 - \innerprod{\bar{U}, A_i Z}) \\
        &=& \sum_{i=1}^{p} 2w_i \innerprod{A_i, \bar{X} Z_1^{\sf T}} (\innerprod{A_i, \bar{X}\bar{X}^{\sf T}-I_m)\bar{U}Z_2^{\sf T}} \\
        &\leq& 2w_+\innerprod{\mathcal{A}(\bar{X}Z_1^{\sf T}), \mathcal{A}((\bar{X}\bar{X}^{\sf T}-I_m)\bar{U}Z_2^{\sf T})}
    \end{eqnarray*}
    Since $\bar{X}^{\sf T}(\bar{X}\bar{X}^{\sf T}-I_m)\bar{U} = 0$, by Lemma~\ref{lem:dec1-0} we have
    \begin{eqnarray*}
        && z_1^{\sf T}(S_{\omega}K - J_{\omega})z_2 \\
        &\leq& 2w_+ \cdot 3\delta_{2k}\norm{\bar{X}Z_1^{\sf T}}_F \norm{(\bar{X}\bar{X}^{\sf T}-I_m)\bar{U}Z_2^{\sf T}}_F \\
        &\leq& 6w_+ \delta_{2k} \norm{Z_1}_F \sqrt{\norm{\bar{U}^{\sf T}(\bar{X}\bar{X}^{\sf T}-I_m)\bar{U}}_F \norm{Z_2^{\sf T}Z_2}_F} \\
        &=& 6w_+\delta_{2k} \sqrt{\norm{\bar{U}^{\sf T}(\bar{X}\bar{X}^{\sf T}-I_m)\bar{U}}_F} \\
        &\leq& 6w_+ \delta_{2k}\sqrt{2k} \norm{\bar{X} - \bar{U}}_F.
    \end{eqnarray*}
    Thus, the spectral norm of this term is upper bounded by $6w_+\delta_{2k}\sqrt{2k} \norm{\bar{X} - \bar{U}}_F$ and finally we have
    \begin{eqnarray*}
        \norm{\myvec(\nabla_Y \mathcal{F}(\bar{U}, Y))}_2 &=& \norm{G_{\omega}(S_{\omega}^{-1}J_{\omega}-I_{nk})v}_2 \\
        &\leq& w_+(1+\delta_{2k})(\sigma_1 \norm{\bar{X}-\bar{U}}_F^2 + \frac{1}{(1-\delta_{2k})w_-}\norm{S_{\omega}K - J_{\omega}}_2 \norm{V}_F) \\
        &\leq& w_+(1+\delta_{2k})(\sigma_1 \norm{\bar{X}-\bar{U}}_F^2 + \frac{\sigma_1 \sqrt{k}}{(1-\delta_{2k})w_-}\norm{S_{\omega}K - J_{\omega}}_2) \\
        &\leq& w_+(1+\delta_{2k})\sigma_1(\norm{\bar{X}-\bar{U}}_F^2 + \frac{\sqrt{k}\cdot 6w_+\delta_{2k}\sqrt{2k}}{(1-\delta_{2k})w_-}  \norm{\bar{X} - \bar{U}}_F )  \\
        &\leq& w_+(1+\delta_{2k})\sigma_1(\norm{\bar{X}-\bar{U}}_F^2 + \frac{6\sqrt{2}\cdot w_+\delta_{2k}k}{(1-\delta_{2k})w_-}  \norm{\bar{X} - \bar{U}}_F ).
    \end{eqnarray*}
    Under the given condition, we can upper bound $\norm{\bar{X} - \bar{U}}$ and $\delta_{2k}$ and we go to the final step as follows:
    \begin{eqnarray*}
        \norm{\myvec(\nabla_Y \mathcal{F}(\bar{U}, Y))}_2
        &\leq& \frac{(1-\delta_{2k})\sigma_k w_-}{2\xi}
        + \frac{(1-\delta_{2k})\sigma_k w_-}{2\xi} \\
        &=& \frac{(1-\delta_{2k})\sigma_k w_-}{\xi}
    \end{eqnarray*}
    Thus, the divergence $\mathcal{D}(Y, Y, \bar{X})$ can be upperbounded by
    \begin{equation}
        \mathcal{D}(Y, Y, \bar{X}) \leq \norm{\myvec(\nabla_Y \mathcal{F}(\bar{U}, Y))}_2  \leq \frac{(1-\delta_{2k})\sigma_k w_-}{\xi}
        \norm{\bar{X}^{(t)}-\bar{U}^{(t)}}_F.
    \end{equation}
\end{proof}

\begin{lem}
    \begin{equation}
        \norm{Y^{(t+0.5)} - V^{(t)}}_F \leq \frac{1}{2w_-(1-\delta_{2k})}\mathcal{D}(Y^{(t+0.5)}, Y^{(t+0.5)}, \bar{X}^{(t)}).
    \end{equation}
    \label{lem:dec1-2}
\end{lem}

\begin{proof}[Proof of Lemma~\ref{lem:dec1-2}]
    Here we utilize the strongly convexity of $\mathcal{F}(X, Y)$ given a orthonormal matrix $X$. By Lemma~\ref{lem:convex}, we have
    \begin{equation}
        \mathcal{F}(\bar{U}, V) \geq \mathcal{F}(\bar{U}, Y) + \innerprod{\nabla_Y \mathcal{F}(\bar{U}, Y), V-Y} + w_-(1-\delta_{2k}) \norm{V-Y}_F^2.
        \label{eq:lem3-1}
    \end{equation}
    Since $V$ minimizes the function $\mathcal{F}(\bar{U}, \hat{V})$, we have $\innerprod{\nabla_Y \mathcal{F}(\bar{U}, V), Y-V} \geq 0$ and thus
    \begin{equation}
    \begin{split}
        \mathcal{F}(\bar{U}, Y) &\geq \mathcal{F}(\bar{U}, V) + \innerprod{\nabla_Y \mathcal{F}(\bar{U}, V), Y-V} + (1-\delta_{2k})w_- \norm{V-Y}_F^2 \\
        &\geq \mathcal{F}(\bar{U}, V) + w_-(1-\delta_{2k}) \norm{V-Y}_F^2.
        \label{eq:lem3-2}
    \end{split}
    \end{equation}
    Add \eqref{eq:lem3-1} and \eqref{eq:lem3-2} we have
    \begin{equation}
        \innerprod{\nabla_Y \mathcal{F}(\bar{U}, Y), Y-V} \geq 2w_-(1-\delta_{2k}) \norm{V-Y}_F^2.
    \end{equation}
    Since $Y$ also minimizes $\mathcal{F}(\bar{X}, \hat{Y})$, we have $\innerprod{\nabla_Y \mathcal{F}(\bar{X}, V), V-Y} \geq 0$ and thus
    \begin{equation}
        \begin{split}
            \innerprod{\nabla_Y \mathcal{F}(\bar{U}, Y)-\nabla_Y \mathcal{F}(\bar{X}, Y), Y-V} &\geq \innerprod{\nabla_Y \mathcal{F}(\bar{U}, Y), Y-V} \\
            &\geq 2w_-(1-\delta_{2k})\norm{V-Y}_F^2.
        \end{split}
    \end{equation}
    Therefore, we have
    \begin{equation}
        \norm{V-Y}_F \leq \frac{1}{2w_-(1-\delta_{2k})}\mathcal{D}(Y, Y, \bar{X})
    \end{equation}
\end{proof}

Given Lemma~\ref{lem:dec1-1} and Lemma~\ref{lem:dec1-2}, we can now bound $\norm{Y^{(t+0.5)} - V^{(t)}}_F$ and thus prove Lemma~\ref{lem:dec1}.
\begin{proof}[Proof of Lemma~\ref{lem:dec1}]
    From Lemma~\ref{lem:dec1-1}, we have
    \[
        \mathcal{D}(Y^{(t+0.5)},Y^{(t+0.5)},\bar{X}^{(t)})
        \leq \frac{(1-\delta_{2k})\sigma_k w_-}{\xi}  \norm{\bar{X}^{(t)}-\bar{U}^{(t)}}_F,
    \]
    and from Lemma~\ref{lem:dec1-2}, we have
    \[ \norm{Y^{(t+0.5)} - V^{(t)}}_F \leq \frac{1}{2w_-(1-\delta_{2k})} \mathcal{D}(Y^{(t+0.5)}, Y^{(t+0.5)}, \hat{X}^{(t)}). \]
    Therefore,
    \begin{eqnarray}
    && \norm{Y^{(t+0.5)} - V^{(t)}}_F  \\
    &\leq& \frac{(1-\delta_{2k})\sigma_k w_-}{\xi} \cdot \frac{1}{2 w_- (1-\delta_{2k})} \norm{\bar{X}^{(t)}-\bar{U}^{(t)}}_F \\
    &=& \frac{\sigma_k}{2\xi} \norm{\bar{X}^{(t)}-\bar{U}^{(t)}}_F
    \end{eqnarray}
\end{proof}

\subsection{Proof of Lemma~\ref{lem:descent}}
\label{proof:descent}

From Lemma~\ref{lem:dec1}, we have
\begin{eqnarray}
    \norm{Y^{(0.5)} - V^{(t)}}_F
    &\leq& \frac{\sigma_k}{2\xi} \norm{\bar{X}^{(t)}-\bar{U}^{(t)}_F} \\
    &\leq& \frac{(1-\delta_{2k})\sigma_k w_-}{2\xi^2(1+\delta_{2k})\sigma_1 w_+} \leq \frac{\sigma_k}{4}, \label{eq:dec2-2}
\end{eqnarray}
where \eqref{eq:dec2-2} is from $\xi>1$. Thus, we can see from Lemma\ref{lem:qr_dec} and we obtain that
\begin{equation}
    \norm{\bar{Y}^{(t+1)} - \bar{V}^{(t+1)}}_F \leq \frac{2}{\sigma_k} \norm{Y^{(0.5)} - V^{(t)}}_F
    \leq \frac{1}{\xi} \norm{\bar{X}^{(t)}-\bar{U}^{(t)}} \leq \frac{(1-\delta_{2k})\sigma_k w_-}{2\xi(1+\delta_{2k})\sigma_1 w_+}.
\end{equation}

\subsection{Proof of Theorem~\ref{thm:weighted}}
\label{proof:weighted}

\begin{lem}
    Suppose that $\delta_{2k}$ satisfies
    \[\delta_{2k} \leq \frac{w_-^2(1-\delta_{2k})^2 \sigma_{k}^4}{48\xi^2 k w_+^2(1+\delta_{2k})^2 \sigma_1^4}.\]
    Then there exists a factorization of $M^* = \bar{U}^{0}V^{(0)\sf T}$ such that $\bar{U}^{(0)} \in \real{m \times k}$ is an orthonormal matrix, and satisfies
    \[\norm{\bar{X}^{(0)} - \bar{U}^{(0)}}_F \leq \frac{w_-(1-\delta_{2k})\sigma_k}{2\xi w_+(1+\delta_{2k})\sigma_1}.\]
    \label{lem:initial}
\end{lem}

\begin{proof}[Proof of Lemma~\ref{lem:initial}]
    \label{proof:initial}
    The initialization step can be regarded as taking a step iterate of singular value projection (SVP) as taking $M^{(t)} = 0$ and the next iterate with the step size $1/(1+\delta_{2k})$ will result $M^{(t+1)} = \bar{X}^{(0)} D^{(0)} \bar{Y}^{(0)}/ (1+\delta_{2k})$, where $\bar{X}^{(0)}$,$D^{(0)}$ and $\bar{Y}^{(0)}$ are from the top $k$ singular value decomposition of $\sum_{i=1}^p b_i A_i$.

    Then, by Lemma~\ref{lem:rip_svp} and the fact that $\varepsilon = 0$, we have
    \begin{equation}
        \left\lVert \mathcal{A}(\frac{\bar{X}^{(0)} D^{(0)} \bar{Y}^{(0)}}{(1+\delta_{2k})}) - \mathcal{A}(M^*)\right\rVert_2^2 \leq 4\delta_{2k} \norm{0 - \mathcal{A}(M^*)}_2^2.
    \end{equation}
    From the $2k$-RIP condition, we have
    \begin{eqnarray*}
        \left\lVert \frac{\bar{X}^{(0)} D^{(0)} \bar{Y}^{(0)}}{(1+\delta_{2k})} \right\rVert
        &\leq& \frac{1}{1-\delta_{2k}}\left\lVert \mathcal{A}(\frac{\bar{X}^{(0)} D^{(0)} \bar{Y}^{(0)}}{(1+\delta_{2k})}) - \mathcal{A}(M^*)\right\rVert_2^2  \\
        &\leq& \frac{4 \delta_{2k}}{1-\delta_{2k}} \norm{\mathcal{A}(M^*)}_2^2 \\
        &\leq& \frac{4 \delta_{2k}(1+\delta_{2k})}{1-\delta_{2k}} \norm{M^*}_F^2
        \leq 6\delta_{2k}\norm{M^*}_F^2.
    \end{eqnarray*}
    Then, we project each column of $M^*$ into the column subspace of $\bar{X}^{(0)}$ and obtain
    \begin{equation*}
        \norm{(\bar{X}^{(0)}\bar{X}^{(0){\sf T}} - I)M^*}_F^2 \leq 6\delta_{2k}\norm{M^*}_F^2.
    \end{equation*}
    We denote the orthonormal complement of $\bar{X}^{(0)}$ as $\bar{X}_{\bot}^{(0)}$. Then, we have
    \begin{equation*}
        \frac{6\delta_{2k}k \sigma_1^2}{\sigma_k^2} \geq \norm{\bar{X}_{\bot}^{(0){\sf T}}\bar{U}^*}_F^2,
    \end{equation*}
    where $\bar{U}^*$ is from the singular value decomposition of $M^*=\bar{U}D\bar{V {\sf T}} $. Then, there exists a unitary matrix $O \in \real{k \times k}$ such that $O^{\sf T}O = I_k$ and
    \[\norm{\bar{X}^{(0)} - \bar{U}^{*}O}_F \leq \sqrt{2}\norm{\bar{X}_{\bot}^{(0){\sf T}}\bar{U}^*}_F \leq 2\sqrt{3\delta_{2k}\frac{\sigma_1}{\sigma_k}}.\]
    By taking the condition of $\delta_{2k}$, we have
    \begin{equation}
        \norm{\bar{X}^{0} - \bar{U}^*}_F \leq \frac{(1-\delta_{2k})\sigma_k w_-}{2\xi(1+\delta_{2k})\sigma_1 w_+}.
    \end{equation}
\end{proof}

\begin{proof}[Proof of Theorem~\ref{thm:weighted}]
    \label{proof:weighted}
    The proof of Theorem~\ref{thm:weighted} can be done by induction. Firstly, we note that Lemma~\ref{lem:initial} ensures that the initial $\bar{X}^{(0)}$ is close to a $\bar{U}^{(0)}$. Then, by Lemma~\ref{lem:qr_dec} we have the following sequence of inequalities for all $T$ iterations:
    \begin{equation}
        \norm{\bar{Y}^{(T)} - \bar{V}^{(T)}}_F \leq \frac{1}{\xi} \norm{\bar{X}^{(T-1)} - \bar{U}^{(T-1)}}_F\leq \cdots \leq \frac{1}{\xi^{2T-1}}\norm{\bar{X}^{(0)} - \bar{U}^{(0)}}_F \leq \frac{(1-\delta_{2k})\sigma_{k}w_-}{2\xi^{2T} (1+\delta_{2k})\sigma_1 w_+}.
    \end{equation}
    Therefore, we can bound the right most term by $\varepsilon / 2$ for any given precision $\varepsilon$. By algebra, we can derive the required number of iterations $T$ as:
    \[ T \geq \frac{1}{2} \log\left(\frac{(1-\delta_{2k})\sigma_k w_-}{2\varepsilon(1+\delta_{2k})\sigma_1 w_+}\right)\log^{-1}\xi.\]
    Similarly, we can also bound $\norm{X^{(T-0.5)} - U^{(T)}}_F$,
    \begin{equation}
        \norm{X^{(T-0.5)} - U^{(T)}}_F \leq \frac{\sigma_k}{2\xi} \norm{\bar{Y}^{(T)} - \bar{V}^{(T)}}_F\leq \frac{(1-\delta_{2k})\sigma_{k}^2 w_-}{4\xi (1+\delta_{2k})\sigma_1 w_+}.
    \end{equation}
    To make it smaller than $\varepsilon \sigma_1 / 2$, we need the number of iterations as
    \[ T \geq \frac{1}{2} \log\left(\frac{(1-\delta_{2k})\sigma_k^2 w_-}{4\varepsilon(1+\delta_{2k})\sigma_1 w_+}\right)\log^{-1}\xi.\]
    Combining all results we have
    \begin{eqnarray}
        \norm{M^{(T)} - M^*}_F
        &=& \norm{X^{(T-0.5)}\bar{Y}^{(T){\sf T}} - U^{(T)}\bar{V}^{(T){\sf T}}}_F \nonumber \\
        &=& \norm{X^{(T-0.5)}\bar{Y}^{(T){\sf T}} - U^{(T)}\bar{Y}^{(T){\sf T}} + U^{(T)}\bar{Y}^{(T){\sf T}} - U^{(T)}\bar{V}^{(T){\sf T}}}_F \nonumber \\
        &\leq& \norm{\bar{Y}^{(T){\sf T}}}_2 \norm{X^{(T-0.5)} - U^{(T)}}_F + \norm{U^{(T)}}_2 \norm{\bar{Y}^{(T)} - \bar{V}^{(T)}}_F \leq \varepsilon.
    \end{eqnarray}
    Here we use the fact that the orthonormal matrix $\bar{V}^{(T)}$ leads to $\norm{\bar{V}^{(T)}}_2 = 1$, and $\norm{M^*}_2 = \norm{U^{(T)} \bar{V}^{(T){\sf T}}}_2 = \norm{U^{(T)}}_2 = \sigma_1$. Now we complete the proof of Theorem~\ref{thm:weighted}.
\end{proof}

\end{document}